\documentclass[11pt, a4paper]{article}

\usepackage[a4paper]{geometry}

\usepackage{xcolor}
\usepackage[english]{babel}

\usepackage{amsmath,amsthm,mathtools,bm}
\usepackage[utf8]{inputenc}
\usepackage{xspace}
\usepackage{url}
\usepackage[numbers]{natbib}
\usepackage{dsfont}

\usepackage{tikz}
\usepackage{pgfplots}
\usepackage{pgfplotstable}
\usetikzlibrary{calc}
\pgfplotsset{compat=1.12}

\newcommand{\xmin}{x_{\min}}
\newcommand{\xmax}{x_{\max}}
\usepackage[algo2e,ruled,vlined]{algorithm2e}
\SetAlgoSkip{}

\allowdisplaybreaks[3]

\newtheorem{theorem}{Theorem}
\newtheorem{lemma}{Lemma}

\newcommand*{\om}{\textsc{OneMax}\xspace}
\newcommand*{\OneMax}{\om}
\newcommand*{\onemax}{\om}

\newcommand{\umdastar}{UMDA$^*$\xspace}

\newcommand{\umda}{UMDA\xspace}

\DeclareMathOperator{\Prob}{Pr}

\newcommand*{\E}{\mathrm{E}}
\newcommand*{\Var}{\mathrm{Var}}

\newcommand*{\bigO}{\mathrm{O}}

\newcommand*{\cumulC}[1]{C_{\geq #1}}
\newcommand{\ie}{i.\,e.\xspace}
\newcommand{\wrt}{w.\,r.\,t.\xspace}
\newcommand{\eg}{e.\,g.\xspace}

\newcommand{\wlo}{w.\,l.\,o.\,g.\xspace}

\newcommand{\indic}[1]{\mathds{1}\{#1\}}
\DeclareMathOperator{\Bin}{Bin}

\newcommand{\N}{\mathds{N}}
\newcommand{\R}{\mathds{R}}
\newcommand{\filt}{\mathcal{F}}
\newcommand{\filtt}{\mathcal{F}_t}

\newenvironment{proofof}[1]{\begin{proof}[Proof of~#1]}{\end{proof}}

\title{Upper Bounds on the Runtime of the Univariate Marginal Distribution Algorithm on OneMax\thanks{An extended abstract 
of this article appeared in the proceedings of the 2017 Genetic and Evolutionary Computation Conference (GECCO~2017) \cite{WittGECCO17}.}}

\author{Carsten Witt\\DTU Compute\\Technical University of Denmark\\
  2800 Kgs. Lyngby\\
	Denmark}

\begin{document}

\maketitle

\begin{abstract}
A runtime analysis of the Univariate Marginal Distribution Algorithm (\umda) is presented on the OneMax function for 
				wide ranges of the parameters $\mu$ and $\lambda$. If $\mu\ge c\log n$   
				for some constant~$c>0$ and $\lambda=(1+\Theta(1))\mu$, 
				a general bound $O(\mu n)$ on the expected runtime is obtained.
				This bound crucially assumes that all marginal probabilities of the algorithm are confined to the interval $[1/n,1-1/n]$. 
				If $\mu\ge c' \sqrt{n}\log n$ for a constant $c'>0$ and $\lambda=(1+\Theta(1))\mu$, the behavior 
				of the algorithm changes and the bound on the expected runtime becomes $O(\mu\sqrt{n})$, 
				which typically even holds if the borders on the marginal probabilities are omitted. 				
				
        The results supplement the recently derived lower bound $\Omega(\mu\sqrt{n}+n\log n)$ 
				by Krejca and Witt (FOGA~2017) and turn out as tight for the two very different choices $\mu=c\log n$ 
				and $\mu=c'\sqrt{n}\log n$. 
				They also improve the previously best known upper bound $O(n\log n\log\log n)$ by Dang and Lehre (GECCO~2015).
				\end{abstract}

\section{Introduction}
\begin{sloppypar}
Estimation-of-distribution algorithms (EDAs, \cite{LarranagaLozanoEDABook}) 
are randomized search heuristics that have 
emerged as a popular alternative to classical 
evolutionary algorithms like Genetic Algorithms. In contrast to the classical approaches, 
EDAs do not store explicit populations of search points but develop a probabilistic model 
of the fitness function to be optimized. Roughly, this model is built by sampling 
a number of search points from the current model and updating it based on the structure 
of the best samples.
\end{sloppypar}

Although many different variants of EDAs (cf.~\cite{HauschildPelikan11}) 
and many different domains are possible, theoretical analysis of EDAs in 
discrete search spaces often considers running time analysis over $\{0, 1\}^n$. The simplest 
of these  EDAs have no mechanism to learn correlations between bits. Instead, 
they store a Poisson binomial distribution, \ie, a probability vector $\bm{ p}$ of $n$ independent 
probabilities, each component $p_i$ denoting the probability that a sampled bit string will have a $1$ at position $i$.

The first theoretical analysis in this setting was conducted by Droste~\citep{Droste2006a}, who 
analyzed the \emph{compact Genetic Algorithm} (cGA), an EDA that only samples two solutions 
in each iteration, on linear functions. Papers considering other EDAs \cite{ChenEtAlCEC09a, ChenEtAlCEC07, ChenEtAlCEC09b, ChenEtAlIEEETEC2010} followed. Also 
iteration-best \emph{Ant Colony Optimization} (ACO), historically classified as a different 
type of search heuristic, can be considered as an EDA and analyzed in the same framework~\cite{Neumann2010a}.

Recently, the interest in the theoretical running time 
analysis of EDAs has increased~\cite{DangLehreGECCO15, FriedrichKK16EDA, FriedrichEtAlISAAC15, SudholtWitt2016, KrejcaWittFOGA2017}. 
 Most of these works derive bounds for a specific EDA on the popular $\om$ function, which counts the number
 of $1$s in a bit string and is considered to be one of the easiest functions with a unique optimum~\cite{Sudholt2012c, Witt2013}. 
In this paper, we follow up on recent work on 
the \emph{Univariate Marginal Distribution Algorithm} (\umda~\cite{MuhlenbeinPaass1996}) on $\om$.

 The \umda is an EDA that samples $\lambda$ solutions in each iteration,  selects $\mu < \lambda$ best solutions, 
and then sets the probability $p_i$ (hereinafter called \emph{frequency}) 
to the relative occurrence of $1$s among these $\mu$ individuals. The algorithm has already been
analyzed some years ago for several artificially designed example 
functions~\cite{ChenEtAlCEC09a, ChenEtAlCEC07, ChenEtAlCEC09b, ChenEtAlIEEETEC2010}. However, none these papers considered 
the most fundamental benchmark 
function in theory, the $\om$ function. 
In fact, the running time analysis of the \umda on the simple $\om$ function has turned out to be rather challenging; the first such result, 
showing the upper bound $O(n\log n\log\log n)$ on its expected running time for certain settings of~$\mu$ and 
$\lambda$, was not published until 2015~\cite{DangLehreGECCO15}.
 
In a recent related study, Wu et~al.~\cite{WuKolonkoMoehringAnalysisCEIEEETEC} 
present the first running time analysis of the cross-entropy method (CE), 
which is a generalization of the \umda and analyze it on \om and another benchmark function. 
Using $\mu=n^{1+\epsilon}\log n$ for some constant~$\epsilon>0$ and 
$\lambda=\omega(\mu)$, they obtain that the running time of CE 
on \om is $O(\lambda n^{1/2+\epsilon/3}/\rho)$ with overwhelming probability, where $\rho$ is a parameter of CE.
Hence, if  $\rho=\Omega(1)$, including the special case $\rho=1$ where CE 
collapses to \umda, a running time bound of $O(n^{3/2+(4/3)\epsilon}\log n)$ 
holds, \ie, slightly above $n^{3/2}$\!.
 Technically, Wu et al.\ use concentration 
bounds such as Chernoff bounds to bound the effect of so-called genetic drift, 
which is also considered in the present paper, 
as well as anti-concentration results, in particular for the Poisson binomial distribution, to  
obtain their statements. All bounds can
hold with high probability only since CE is formulated without so-called borders on the frequencies.

Very recently, these upper bounds were supplemented 
by a general lower bound of the kind $\Omega(\mu\sqrt{n}+n\log n)$ \cite{KrejcaWittFOGA2017}, proving that 
the \umda cannot be more efficient than simple evolutionary algorithms on this function, at least if $\lambda=(1+\Theta(1))\mu$. As the upper bounds 
due to~\cite{DangLehreGECCO15} and the recent 
lower bounds were apart by a factor of $\Theta(\log\log n)$, it was an open problem 
to determine the asymptotically best possible running time of the \umda on \om.

In this paper, we close this gap and show that the \umda can optimize \om in expected time $O(n\log n)$ 
for two very different, carefully chosen values of $\mu$, always assuming that
 $\lambda=(1+\Theta(1))\mu$. In fact, we obtain two general upper bounds depending on 
$\mu$. If $\mu\ge c\sqrt{n}\log n$, 
where $c$ is a sufficiently large constant, the first upper bound is $O(\mu\sqrt{n})$. This 
bound exploits that all $p_i$ move more or less steadily to the largest possible value  and that 
with high probability there are no frequencies that ever drop below~$1/4$. Around $\mu=\Theta(\sqrt{n}\log n)$, 
there is a phase transition in the behavior of the algorithm. With smaller~$\mu$, the stochastic movement of the 
frequencies is more chaotic and many frequencies will hit the lowest possible value during the 
optimization. Still, the expected optimization time is $O(\mu n)$ for $\mu\ge c'\log n$ and  
a sufficiently large constant~$c'>0$ if all 
frequencies are confined to the interval $[1/n,1-1/n]$, 
as typically done in EDAs. If frequencies are allowed to drop to $0$, the algorithm will typically have
infinite optimization time below the phase transition point $\mu\sim\sqrt{n}\log n$, whereas 
it typically will be efficient above.

Interestingly, Dang and Lehre \cite{DangLehreGECCO15} used $\mu=\Theta(\ln n)$, \ie, a value below the
phase transition  to obtain their $ O(n\log n\log\log n)$ bound. This region turns out to 
be harder to analyze than the region above the phase transition, at least with our techniques. 
However, our proof also follows an approach being widely different from \cite{DangLehreGECCO15}. 
There the so-called level based theorem, a very general upper bound technique, is applied 
to track the stochastic behavior of the best-so-far \om-value. While this gives a rather short and elegant 
proof of the upper bound $ O(n\log n\log\log n)$, the generality of the technique does not give much insight into 
how the probabilities $p_i$ of the individuals bits develop over time. We think that it is crucial to understand
 the working principles of the algorithm thoroughly and present a detailed analysis of the stochastic process 
at bit level, as also done in many other running time 
analyses of EDAs \cite{FriedrichKK16EDA, FriedrichEtAlISAAC15,SudholtWitt2016,KrejcaWittFOGA2017}.

This paper is structured as follows: in Section~\ref{sec:prelims}, we introduce the setting 
we are going to analyze and summarize some tools from probability theory that are used 
throughout the paper. In particular, a new negative drift theorem is presented. It generalizes 
previous formulations by making  
milder assumptions on steps in the direction of the drift than on steps against the drift. 
 In this section, we also give a detailed analysis of the update rule of the \umda, which results 
in a bias of the frequencies~$p_i$ towards higher values. These techniques are presented 
for the \om-case, but contain some general insights that may be useful in analyses of 
different fitness functions. 
In Section~\ref{sec:upperBound}, we prove the upper bound for the 
case of $\mu$ above the phase transition point $\Theta(\sqrt{n}\log n)$. The 
case of $\mu$ below this point is dealt with in Section~\ref{sec:upperBoundTwo}. 
We finish with some conclusions. The appendix gives a self-contained proof 
of the new drift theorem.

\paragraph{Independent, related work.} Very recently, Lehre and Nguyen \cite{LehreNguyenGECCO17} 
independently obtained the upper bound $O(\lambda n)$ for $c\log n\le \mu = O(\sqrt{n})$ and 
$\lambda=\Omega(\mu)$ using a 
refined application of the so-called level-based method. Our approach also covers larger $\mu$ (but requires 
$\lambda = \mu(1+\Theta(1))$) and is 
technically different.

\section{Preliminaries}
\label{sec:prelims}

\begin{sloppypar}We consider the so-called \emph{Univariate Marginal Distribution Algorithm} (\umda~\cite{MuhlenbeinPaass1996}) 
in Algorithm~\ref{alg:UMDA} that maximizes the pseudo-Boolean function~$f$. Throughout this paper, we have 
$f\coloneqq \om$, where, for all $x=(x_1,\dots,x_n) \in \{0, 1\}^n$,
\[
\om(x) = \sum_{i = 1}^{n} x_i.
\]
Note that the unique maximum is the all-ones bit string. However, a more general version can 
be defined by choosing an arbitrary optimum $a \in \{0, 1\}^n$ and defining, for all $x \in \{0, 1\}^n$,
$
    \om_a(x) = { n - d_{\mathrm{H}}(x, a)},
$
where $d_{\mathrm{H}}(x, a)$ denotes the Hamming distance of the bit strings $x$ and $a$. 
Note that $\om_{1^n}$ is equivalent to the original definition of $\om$. Our analyses hold true 
for any function $\om_a$, with $a \in \{0, 1\}^n$, due to symmetry of the \umda's update rule.\end{sloppypar}

\begin{algorithm2e}
    $t \gets 0$, 
    $p_{t,1} \gets p_{t,2} \gets \cdots \gets p_{t,n} \gets \tfrac{1}{2}$\;
    \While{\emph{termination criterion not met}}
    {
        $P_t \gets \emptyset$\;
        \For{$j \in \{1, \ldots, \lambda\}$}
        {
            \For{$i \in \{1, \ldots, n\}$}
            {
                $x^{(j)}_{t, i} \gets 1$ with prob.\ $p_{t,i}$ and $x^{(j)}_{t,i} \gets 0$ with prob.\ $1 - p_{t,i}$\;
            }
            $P_t \gets P_t \cup \{x^{(j)}_t\}$\;
        }
        Sort individuals in $P$ descending by fitness (such that $f(x^{(1)}_t) \ge \dots\ge f(x^{(\mu)}_t)$), breaking ties uniformly at random\;
        \For{$i \in \{1,\dots,n\}$}
        {
            $p_{t + 1,i} \gets \frac{\sum_{j = 1}^{\mu} x^{(j)}_{t,i}}{\mu}$\;
            $\bigl[R\bigr]$ Restrict $p_{t+1,i}$ to be within $[\tfrac{1}{n}, 1 - \tfrac{1}{n}]$\;
        }
        $t \gets t+1$\;
    }
    \caption{Univariate Marginal Distribution Algorithm (\umda); algorithm \umdastar is obtained if the line indexed $\bigl[R\bigr]$ is omitted.}
    \label{alg:UMDA}
\end{algorithm2e}

We call bit strings \emph{individuals} and their respective $\om$-values \emph{fitness}.

The \umda does not store an explicit population but does so implicitly, as usual in 
EDAs. For each of the $n$ different bit positions, 
it stores a rational number $p_i$, which we call \emph{frequency}, determining how likely it is 
that a hypothetical individual would have a $1$ at this position. In other words, the \umda stores 
a probability distribution over $\{0, 1\}^n$. The starting distribution samples according to 
the uniform distribution, \ie, 
 $p_i=1/2$ for $i\in\{1,\dots,n\}$. 

In each so-called generation~$t$, the \umda samples $\lambda$ individuals such that each individual has a $1$ 
at position $i$, where $i \in \{1, \dots, n\}$ with probability $p_{t,i}$, independent of all the other 
frequencies. Thus, the number of $1$s is sampled according to a Poisson binomial distribution with probability 
vector $\bm{p}_{t}=(p_{t,i})_{i \in \{1, \dots, n\}}$.

After sampling $\lambda$ individuals, $\mu$ of them with highest fitness are chosen, 
breaking ties uniformly at random (so-called \emph{selection}). Then, for each position, 
the respective frequency is set to the relative occurrence of $1$s in this position. That is, 
if the chosen $\mu$ best individuals have $x$ $1$s at position~$i$ among them, the frequency $p_i$ will be 
updated to $x/\mu$ for the next iteration. Note that such an update allows large 
jumps like, \eg, from $(\mu - 1)/\mu$ to $1/\mu$.

If a frequency is either $0$ or $1$, it cannot change anymore since then all values at this position will be 
either $0$ or $1$. To prevent the \umda from getting stuck in this way, we narrow the interval of possible 
frequencies down to $[1/n, 1 - 1/n]$ and call $1/n$ and $1-1/n$ the \emph{borders} for the frequencies. Hence, 
there is always a chance of sampling $0$s and $1$s for each position. This is a common approach used by other EDAs 
as well, such as the cGA or ACO algorithms (cf.\ the references given in the introduction). We also consider a variant of the \umda 
called \umdastar where the borders 
are not used. That algorithm will typically not have finite expected running time; however, it might still be efficient 
with high probability if it is sufficiently unlikely that frequencies get stuck at bad values.

Overall, we are interested in upper bounds on the \umda's expected number of \emph{function evaluations} on $\om$ 
until the optimum is sampled; this number is typically called \emph{running time} or \emph{optimization time}. 
Note that this equals $\lambda$ times the expected number of generations until the optimum is sampled.

In all of our analyses, we  assume that $\lambda = (1 + \beta)\mu$ for some arbitrary constant $\beta > 0$ and 
use $\mu$ and $\lambda$ interchangeably in asymptotic notation. Of course, we 
could also choose $\lambda = \omega(\mu)$ but then each generation would be  more expensive. Choosing $\lambda = \Theta(\mu)$ 
lets us basically focus on the minimal number of function evaluations per generation, as $\mu$ of them are at least 
needed to make an update.

\subsection{Useful Tools from Probability Theory}

We will see that the number of $1$s sampled by the UMDA at a certain 
position is binomially distributed with the frequency as success probability. 
In our analyses,  
we will therefore often have to bound the tail of binomial and related distributions.   
To this end, 
many classical techniques such as Chernoff-Hoeffding bounds exist. The following 
version, which includes the knowledge of the variance, is particularly handy to use.

\begin{lemma}[\cite{McDiarmid1998}] 
\label{lem:mcdiarmid}
If $X_1,\dots,X_n$ are independent, and $X_i-\E(X_i)\le b$ for 
$i\in\{1,\dots,n\}$, then for $X\coloneqq X_1+\dots+X_n$ and any $d\ge 0$ it holds 
that \[\Prob(X-\E(X)\ge d) \le e^{-\frac{d^2}{2\sigma^2(1+\delta/3)}},\]
where $\sigma^2 \coloneqq \Var(X)$ and $\delta \coloneqq bd/\sigma^2$.
\end{lemma}

The following lemma describes a result regarding the Poisson binomial distribution which 
we find very intuitive. However, as we did not find a sufficiently related result in the literature, 
we give a self-contained 
proof here. Roughly, the lemma considers a chunk of the distribution 
around the expected value 
whose joint probability is a constant less than~$1$ and then argues that every 
point in the chunk has a probability that is at least inversely proportional to the variance. 
See Figure~\ref{fig:illustration-1-over-sigma} for an illustration.
\begin{lemma}
\label{lem:at-least-1-over-sigma}
Let $X_1,\dots,X_n$ be independent Poisson trials. Denote $p_i=\Prob(X_i=1)$ for $i\in\{1,\dots,n\}$, 
$X\coloneqq \sum_{i=1}^n X_i$, 
$\mu\coloneqq \E(X)=\sum_{i=1}^n p_i$ and $\sigma^2\coloneqq \Var(X)=\sum_{i=1}^n p_i(1-p_i)$.
Given two constants $\ell,u\in(0,1)$ such that $\ell+u<1$, 
let $k_\ell\coloneqq \min\{i\mid \Prob(X\le i)\ge \ell\}$ and $k_u
\coloneqq \max\{i\mid \Prob(X\ge i)\ge u\}$. Then 
it holds that $\Prob(X=k)=\Omega(\min\{1,1/\sigma\})$ for all $k\in\{k_\ell,\dots,k_u\}$, where 
the $\Omega$-notation is with respect to~$n$.
\end{lemma}
 
\begin{figure}
\centering
\begin{tikzpicture}[xscale=1]
\clip(1,-0.5) rectangle +(8,3.63);
\begin{axis}[
  no markers, domain=2:18, samples=20,
  axis lines*=left, xlabel={}, ylabel=distribution,
  every axis y label/.style={at=(current axis.above origin),anchor=south},
  every axis x label/.style={at=(current axis.right of origin),anchor=west},
  height=4cm, width=12cm,
  xtick=\empty, ytick=\empty,
  extra x ticks={6.85, 12.09},
  extra x tick style={grid=major,xticklabel shift=0},
  extra x tick labels={},
  enlargelimits=false, clip=false, axis on top,
  grid = none,
  hide y axis,
	bar shift = 0cm,
	ybar , bar width = 1, 
	declare function={binom(\n,\p) = \n!/(x!*(\n-x)!)*\p^x*(1-\p)^(\n-x);
    },
		declare function={binoma(\n,\p) = 1.15*\n!/(x!*(\n-x)!)*\p^x*(1-\p)^(\n-x);
    },
    xtick style={draw=none},
    yticklabel style={
        /pgf/number format/fixed,
        /pgf/number format/fixed zerofill,
        /pgf/number format/precision=2,
    }
  ]
	\addplot[forget plot, domain=0:20] {binom(20,0.48)};
	\begin{scope}
	\clip(0,0) rectangle (5.8,1.2	);
	\addplot[fill=black!50,domain=0:20, forget plot] {binom(20,0.48)};
	\end{scope}
	\begin{scope}
	\clip(0,0) rectangle (0.0,1.2	);
		\addplot[forget plot,fill=black!50, domain=0:20] {binoma(20,0.48)};
	\end{scope}
	
	\begin{scope}
	\clip(13.2,0) rectangle +(7,2	);
	\addplot[fill=black!50,domain=0:20] {binom(20,0.48)};
	\end{scope}
  
  \draw(axis cs:  5.65,1.8) node {tail $\ge \ell$};
	\draw(axis cs:  13.3,1.8) node {tail $\ge u$};
	\node(A) at (axis cs: 7,-0.4) {};
	\node(B) at (axis cs: 13,-0.4) {};	
	\node(C1) at (axis cs: 9.5,-0.2) {$\mu$};
	\node(C2) at (axis cs: 6.4,-0.2) {$k_\ell$};
	\node(C3) at (axis cs: 12.6,-0.2) {$k_u$};
	\draw[->](C1) -- node[pos=0.48,above=-0.8]{\tiny maximize} (C2) ;
	\draw[->](C1) -- node[pos=0.5,above=-0.8]{\tiny maximize} (C3) ;
	\node at (axis cs: 9.4,2.3) {probability $\Omega(\min\{1,1/\sigma\})$ each};
	\node(A0) at (axis cs: 6.3,0.7) {};
	\node(A1) at (axis cs: 7.3,1.1) {};
	\node(A2) at (axis cs: 8.4,1.58) {};
	\node(A3) at (axis cs: 9.5,1.75) {};
	\node(A4) at (axis cs: 10.6,1.6) {};
	\node(A5) at (axis cs: 11.7,1.2) {};
	\node(A6) at (axis cs: 12.7,0.72) {};
	\node(B1) at (axis cs: 9.5,2.1) {};
	\node(B2) at (axis cs: 9.5,2.2) {};
	\node(B3) at (axis cs: 9.6,2.15) {};
	\node(B4) at (axis cs: 9.4,2.15) {};
	\draw[->,node distance = 0.0] (B1) edge[bend right=45]  (A1);
	\draw[->,node distance = 0.0] (B1) edge[bend right=45]  (A0);
	\draw[->] (B3) edge[bend right=20] (A2);
	\draw[->] (B2) -- (A3);
	\draw[->] (B4) edge[bend left=20] (A4);
	\draw[->] (B1) edge[bend left=45] (A5);
	\draw[->] (B1) edge[bend left=45] (A6);
	\end{axis}
\end{tikzpicture}
\caption{Illustration of Lemma~\ref{lem:at-least-1-over-sigma}.}
\label{fig:illustration-1-over-sigma}
\end{figure}
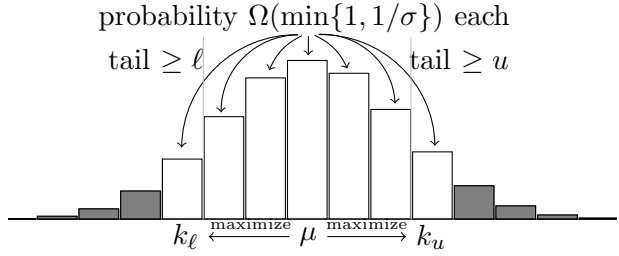

\begin{proof}
To begin with, we note that $k_\ell\le k_u$.  This holds since by assumption $\Prob(X<k_\ell)<\ell$ 
and $\Prob(X>k_u)<u$, hence $\Prob(k_\ell \le X \le k_u)\ge 1-\ell-u > 0$, using $\ell+u<1$. If $k_\ell>k_u$
happened, 
we would obtain a contradiction.

We first handle the case $\sigma=o(1)$ separately. This implies $\Prob(X-\E(X)\ge 1/2)=o(1)$ and analogously 
$\Prob(X-\E(X)\le -1/2)=o(1)$. Namely, if we had 
$\Prob(X-\E(X)\ge 1/2)=\Omega(1)$, then $\E((X-\E(X))^2)=\Omega(1)$, contradicting the assumption $\sigma=o(1)$; 
analogously for the other inequality.
 
Let $[\E(X)]$ be the integer closest to~$\E(X)$, which is unique since, as argued in the previous paragraph, 
$\E(X)-[\E(X)] = 1/2$ would contradict 
$\sigma=o(1)$. We assume 
$[\E(X)]=\lceil \E(X)\rceil$ and note that the case $[\E(X)]=\lfloor \E(X)\rfloor$ is analogous. 
From the previous paragraph, we now obtain that 
 $\Prob(X\le \lfloor\E(X)\rfloor)=o(1)$ and therefore $\Prob(X\ge \lceil\E(X)\rceil)=1-o(1)$. Moreover, 
again using $\sigma=o(1)$, we also obtain $\Prob(X\ge \lceil\E(X)\rceil+1)=o(1)$.
Since $\ell$ and 
$u$ are positive constants less than~$1$, 
it immediately follows that $k_\ell=k_u=[\E(X)]$ and $\Prob(X=k_\ell)=1-o(1)=\Omega(1)$. 

%
%

In the following, we assume $\sigma=\Omega(1)$. We use that 
$k_\ell\le \lfloor\E(X)\rfloor$ or $k_u\ge \lceil\E(X)\rceil$ (or both) 
must hold; otherwise, since $\lceil\E(X)\rceil\le \lfloor\E(X)\rfloor+1$, we 
would contradict the fact $k_\ell\le k_u$. 
Hereinafter we consider the case $k_\ell\le \lfloor\E(X)\rfloor$ and note that other case is symmetrical.
 We start by proving 
$\Prob(X=k_\ell)=\Omega(1)$. To this end, we recall 
the unimodality of the Poisson binomial distribution function, more precisely 
$\Prob(X=i)\le \Prob(X=i+1)$ for $i\le \lfloor\E(X)\rfloor-1$ and $\Prob(X=i)\ge \Prob(X=i+1)$ for 
$i\ge \lceil\E(X)\rceil$ \cite{Samuels1965}. Hence, 
 denoting  $\alpha\coloneqq \Prob(X=k_\ell)$, we have 
$\Prob(X=i)\le \alpha$ for all $i\le k_\ell$. It follows  
 that $\Prob(X\le k_\ell-\ell/(2\alpha))\ge \ell/2=\Omega(1)$ 
since $\Prob(X\le k_\ell)\ge \ell$ by definition. We remark (but do not use) that this also implies 
a lower bound on $k_\ell$. 

If $\alpha=o(1/\sigma)$ held, the fact that $\Prob(X\le k_\ell-\ell/(2\alpha))=\Omega(1)$ 
would imply $\sqrt{\Var(X)}=\Omega(1/\alpha)=\omega(\sigma)$, 
contradicting our assumption $\sqrt{\Var(X)}=\sigma$. Hence, 
$\Prob(X=k_\ell)=\Omega(1/\sigma)$. Again using the monotonicity of 
the Poisson binomial distribution, we have
$\Prob(X=i)=\Omega(1/\sigma)$ for all $i\in\{k_\ell,\dots,\lfloor\E(X)\rfloor\}$. 
If $k_u\le \lfloor\E(X)\rfloor$, this already proves 
$\Prob(X=i)=\Omega(1/\sigma)$ for all $i\in\{k_\ell,\dots,k_u\}$ and nothing 
is left to show. Otherwise 
the bound follows for the remaining $i\in\{\lceil \E(X)\rceil,\dots,k_u\}$ 
by a symmetrical argument, more precisely by first showing that 
$\Prob(X=k_u)=\Omega(1/\sigma)$ and then 
using that $\Prob(X=i)\ge \Prob(X=i+1)$ for $i\ge \lceil \E(X)\rceil$.  
%
\qed\end{proof}

As mentioned, we will study how the frequencies associated with single bits evolve over time. 
To analyze the underlying stochastic process, the following theorem will be used. It generalizes 
the so-called \emph{simplified drift theorem with scaling} from \cite{OlivetoWittTCS15}. The crucial 
relaxation is that the original version demanded an exponential decay \wrt jumps in both 
directions, more precisely the second condition below was on $\Prob(\lvert X_{t+1}-X_{t}\rvert\ge jr)$. 
 We now only have sharp demands on jumps in the undesired direction while 
there is a milder assumption (included in the first item) on jumps in the desired direction. Roughly 
speaking, if constants in the statements do not matter, the previous version of the drift theorem 
is implied by the current one 
as long as $r=O(1)$. 

The theorem uses the notation $\E(X\mid \filtt;A)$ for filtrations $\filtt$ and events~$A$ to denote 
the expected value $\E(X\mid \filtt)$ in the conditional probability space on event~$A$. If $A$ is not a null event, 
then $\E(X\mid \filtt;A)\ge \epsilon$ is equivalent to $\E(X-\epsilon;A\mid \filtt)\ge 0$, where the 
notation ``$;A$'' just denotes the multiplication with $\indic{A}$; in fact 
the notation $\E(X;A\mid \filtt)$ is often used in the literature, \eg, by Hajek~\cite{Hajek1982}. 
Additionally $X\preceq Y$ denotes that $X$ is stochastically at most as large as~$Y$.
The proof of 
the theorem is given in the appendix.

\begin{theorem}[Generalizing \cite{OlivetoWittTCS15}]\label{theo:negative-drift-scaling-2017}
  Let $X_t$, $t\ge 0$, be real-valued random variables describing a
	stochastic process over some state space, adapted to a filtration $\filtt$. Suppose 
  there exist an interval $[a,b]\subseteq \R$
   and, possibly depending on
  $\ell\coloneqq b-a$, a drift bound $\epsilon\coloneqq \epsilon(\ell)>0$, 
	a typical forward jump factor $\kappa\coloneqq \kappa(\ell)>0$,
	a scaling factor $r\coloneqq r(\ell)>0$ 
	as well as a sequence of functions $\Delta_t\coloneqq \Delta_t(X_{t+1}-X_t)$ satisfying 
	$\Delta_t\preceq X_{t+1}-X_t$ 
   such that  
  for all $t\ge 0$ the following three conditions hold:
  \begin{enumerate}
  \item 
	$\E(\Delta_t\cdot\indic{\Delta_t\le \kappa\epsilon  } \mid \filtt \,;\, a< X_t <b) \ge \epsilon$,
  \item $\Prob(\Delta_t\le -jr \mid \filtt  \,;\, a< X_t)  \le  e^{-j}$ for all $j\in \N$, 
  \item
  $  \lambda \ell \ge 2\ln(4/(\lambda \epsilon))$, 
	where $\lambda\coloneqq \min\{1/(2r),\epsilon/(17r^2),1/(\kappa\epsilon)\}$. 
  \end{enumerate}
  Then for   $T^*\coloneqq \min\{t\ge
  0 \mid X_t\le a \}$ it holds that $\Prob(T^*\le
  e^{\lambda\ell/4}\mid \filt_0 \,;\, X_0\ge b) = O(e^{-\lambda\ell/4})$.
\end{theorem}

To derive upper bounds on hitting times for an optimal state, drift analysis is used, 
in particular in scenarios where the drift towards the optimum is not state-homogeneous.  Such a drift is called \emph{variable} in the literature. 
A clean form of a variable drift theorem, generalizing 
previous formulations from \cite{Johannsen10} and \cite{MitavskiyVariable}, 
was presented in \cite{RoweSudholtTCS2014}. The following formulation 
has been proposed in \cite{LehreWittISAAC14}.

\begin{theorem}[Variable Drift, Upper Bound] 
\label{theo:variable-drift}
Let $(X_t)_{t\in\N_0}$, be a stochastic process,  adapted to a filtration $\filtt$,  over some state space  $S\subseteq \{0\}\cup [\xmin,\xmax]$, where $\xmin>0$.  
Let  $h(x)\colon [\xmin,\xmax]\to\R^+$ be a monotone increasing function such that 
$1/h(x)$ is integrable on $[\xmin,\xmax]$ and 
$\E(X_t-X_{t+1} \mid \filtt) \ge h(X_t)$ if $X_t\ge \xmin$.
 Then it holds for the first hitting time 
$T\coloneqq \min\{t\mid X_t=0\}$ that 
\[
\E(T\mid \filt_0) \le 
\frac{\xmin}{h(\xmin)} + \int_{\xmin}^{X_0} \frac{1}{h(x)} \,\mathrm{d}x.
\] 
\end{theorem}

Finally, we need the following lemma in our analysis of the impact of the so-called $2$nd-class individuals in Section~\ref{sec:first-and-second-class}. Its statement is 
very specific and tailored to our applications. Roughly, the intuition is to show that $\E(\min\{C,X\})$ is not much less than 
$\min\{C,\E(X)\}$ for $X\sim \Bin(D,p)$ and $D\ge C$. Here and in the following, we write $\Bin(a,b)$ to denote the binomial distribution with 
parameters~$a$ and~$b$.

\begin{lemma}
\label{lem:concave-expec}
Let $X\sim \Bin(D,p)$. Let $C\in\{1,\dots,D\}$. 
Then 
\[\E(\min\{C,X\})\ge Cp + \frac{1}{4} p (1-p) \min\{C,D-C\}.\]
\end{lemma}

\begin{proof}
We start by deriving a general lower bound on the expected value of $\min\{C,X\}$. The idea is to 
decompose the random variable $X$, which is a sum of $D$ independent trials, 
into the the first $C$ and the remaining $D-C$ trials. 
Let $Y\sim\Bin(C,p)$  and 
$Z\sim\Bin(D-C,p)$. Hence, $X=Y+Z$ and, since $Z$ is independent of $Y$, we have 
$\min\{C,X\} = \min\{C,Y\} + \min\{C-Y, Z\} =  
Y + \min\{C-Y,Z\}$. We also note that $\E(Y)=Cp$ and $\E(Z)=(D-C)p$. 

Assume that  for some $k<C$ and some $p^*>0$, we know that $\Prob(Y\le k)\ge p^*$. Then  
 by the law of total probability 
\begin{align}
\E(\min\{C,X\}) & \ge \bigl(\E(Y\mid Y\le k) + \E(\min\{C-k, Z\}) \bigr) \Prob(Y\le k) \notag \\ & \qquad + 
\E(Y\mid Y > k) \Prob(Y > k) \notag \\
&  = \E(Y) + \E(\min\{C-k, Z\})  \Prob(Y\le k) \notag \\
& \ge  \E(Y) + p^*\cdot \E(\min\{C-k, Z\}) ).
\label{eq:lemconcaveexp-1}
\end{align}
In the following, we will distinguish between two cases with respect to~$p$, 
in which appropriately chosen  pairs $(k,p^*)$ 
 imply the lemma. 

\textbf{Case 1: }$p\le 1-2/C$. 
Hence $C(1-p)\ge 2$, which implies $\lfloor C(1-p)/2\rfloor\ge 1$. Therefore, 
$\lceil \E(Y)\rceil \le \E(Y)+1 \le Cp + \lfloor C(1-p)/2\rfloor$.
We apply the bound 
$\Prob(Y\le \lceil\E(Y)\rceil)\ge 1/2$, which is equivalent to the well-known 
bound $\Prob(A\ge \lfloor\E(A)\rfloor)\ge 1/2$ that holds for 
all binomially distributed random variables~$A$ \cite{KassBuhrmanMeanMedianMode}.  Hence, 
\[
\Prob(Y\le Cp + \lfloor C(1-p)/2\rfloor)\ge \frac{1}{2}. 
\]
Using \eqref{eq:lemconcaveexp-1} with $p^*\coloneqq 1/2$ and $k\coloneqq Cp + \lfloor C(1-p)/2\rfloor$ we conclude 
\begin{align*}
\E(\min\{C,X\}) & \ge Cp + \frac{1}{2}  \cdot \E(\min\{C(1-p)-\lfloor C(1-p)/2\rfloor,Z\}) \\ 
 & = Cp + \frac{1}{2}  \cdot \E(\min\{\lceil C(1-p)/2\rceil,\Bin(D-C,p)\}) \\
& \ge Cp+ \frac{1}{2}  \cdot \E(\Bin(\min\{\lceil C(1-p)/2\rceil,D-C\},p)) \\
& = Cp + \frac{1}{2} p \min\{\lceil  C(1-p)/2\rceil, D-C\} \\ 
& \ge Cp+ \frac{1}{2} p \min\{C(1-p)/2,(D-C)(1-p)/2\}  \\
&  = Cp + \frac{1}{4}p(1-p) \min\{C,D-C\},
\end{align*}
where the second inequality exploits that $\Bin(A,p)$ is stochastically larger 
than $\Bin(B,p)$ for all $B\le A$, and clearly $\Bin(B,p)\le B$. The third inequality 
uses that  $(1-p)/2\le 1$ and the final equality exploits that $1-p$ is non-negative. 
Hence, the lemma holds in this case.

\textbf{Case 2: }$p > 1-2/C$. The aim is to show 
that $\Pr(Y\le C-1)\ge C(1-p)/3$. In the subcase that $C\le 3$, we clearly have $\Prob(Y\le C-1)\ge 1-p \ge C(1-p)/3$. 
If $C\ge 4$,  we work with $q\coloneqq 1-p\le 2/C\le 2$ and 
note that $\Pr(Y\le C-1) = 1 - p^C = 1-(1-q)^C$. Now, 
\[
1-(1-q)^C \ge 1-e^{-qC} \ge 1-\left(1-\frac{qC}{3}\right) = \frac{qC}{3} = \frac{C(1-p)}{3},
\]
where the first inequality uses $e^{x}\ge 1+x$ for $x\in\R$ and the second 
$e^{-x}\le 1-\frac{x}{3}$ for $x\le 2$. Hence, $\Pr(Y\le C-1)\ge C(1-p)/3$ for all 
$C\in\{1,\dots,D\}$. 
Using~\eqref{eq:lemconcaveexp-1} 
with $k=C-1$ and $p^*\coloneqq \frac{C(1-p)}{3}$ and proceeding similarly to Case~1, we obtain
\begin{align*}
\E(\min\{C,X\}) & \ge Cp + \frac{C(1-p)}{3}  \cdot \E(\min\{1,Z\}) \\
& \ge Cp + \frac{C(1-p)}{3} \E(\Bin(\min\{1,D-C\},p)) \\
& = Cp + \frac{C(1-p)}{3} p \min\{1,D-C\} \\ 
& = Cp + \frac{p(1-p)}{3}  \min\{C,C(D-C)\} \\
& \ge Cp + \frac{p(1-p)}{3}  \min\{C,D-C\} ,
\end{align*}
where the last inequality used that $C\ge 1$ and the previous equality used that $C$ is non-negative. 
This concludes Case~2 and proves the lemma.
\qed\end{proof}

\subsection{On the Stochastic Behavior of Frequencies}
\label{sec:first-and-second-class}
To bound the expected running time of \umda and \umdastar, it is crucial to understand 
how the $n$ frequencies associated with the bits evolve over time. The symmetry of the fitness function 
\onemax implies that each frequency evolves in the same way, but not necessarily 
independently of the others. Intuitively, many frequencies should be close to their upper 
border for making it sufficiently likely to sample the optimum, \ie, the all-$1$s string.

To understand the stochastic process on a frequency, it is useful to consider the \umda 
without selection for a 
moment. More precisely, assume that each of the $\lambda$ offspring 
has the same probability of being selected as one of the $\mu$ individuals determining 
the frequency update. Then the frequency describes a random walk that 
is a martingale, \ie, in expectation it does not change. With \onemax, individuals with 
higher values are more likely to be among the $\mu$ updating individuals. However, since 
only the accumulated number of $1$-bits per individual matters for selection, 
a single frequency may still decrease even if the step leads to an increase of the best-so-far seen 
\onemax value. We will spell out the  
bias due to selection  in the remainder of this section.

In the following, we consider an arbitrary but fixed bit position~$j$ and denote by $p_t\coloneqq p_{t,j}$ its frequency at time~$t$. 
Moreover, let $X_t$, where $0\le X_t\le \mu$, be the number of ones at position~$j$ among the $\mu$ offspring 
selected to compute~$p_t$. Then $p_t=\mathrm{cap}_{1/n}^{1-1/n} (X_t/\mu)$, where $\mathrm{cap}_\ell^h (a) \coloneqq 
\max\{\min\{a,h\},l\}$  caps frequencies at their borders.

\subsubsection*{Ranking, 1st-class individuals, 2nd-class individuals and candidates}
Consider the fitness of all individuals sampled during one generation of the \umda \wrt $n - 1$ 
bits, \ie, all bits but bit $j$. Assume that the individuals are sorted 
in levels decreasingly by their fitness; each individual having a unique rank, where ties 
are broken arbitrarily. Level $n - 1$ is called 
the topmost, and level $0$ the lowermost. 
Let $w^+$ be the level of 
the individual with rank $\mu$, and let $w^-$ be the level of the individual with rank $\mu + 1$. 
Since bit~$j$ has not been considered so far, its \om-value can potentially increase each 
individual’s level by~$1$. Now assume that $w^+ = w^- + 1$. Then, individuals from level 
$w^-$ can end up with the same fitness as individuals from level $w^+$, once bit $j$ has been sampled. 
Thus, individuals from 
level $w^+$ were still prone to selection.

Among the $\mu$ individuals chosen during selection, we distinguish between two different types: 
$1$st-class and $2$nd-class individuals. $1$st-class individuals are those which have so many $1$s at the $n-1$ 
other bits such that they 
had to be chosen 
during selection no matter which value bit $j$ has. The remaining of the $\mu$ selected individuals are 
the $2$nd-class individuals; they had to compete with other individuals for selection. Therefore, their bit value 
at position~$j$ is biased towards~$1$ compared to $1$st-class individuals. Note that $2$nd-class individuals 
can only exist if $w^+ \leq w^- + 1$, since in this case, individuals 
from level $w^-$ can still be as good as individuals from level $w^+$ after sampling bit~$j$.

Given $X_t$,  let $C_{t+1}^*$ denote the number of $2$nd-class individuals in 
generation $t + 1$. Note that the total number of $1$s at position~$j$ in the $1$st-class individuals 
during generation $t + 1$ follows a binomial distribution with success probability $p_t=X_t/\mu$, assuming 
$X_t/\mu$ is within the interval $[1/n,1-1/n]$. 
Since we have $\mu - C^*_{t+1}$ $1$st-class individuals, the distribution of the number of $1$s in  
these follows $\Bin(\mu - C^*_{t+1}, X_t/\mu)$.

We proceed by analyzing the number of $2$nd-class individuals and how they bias the 
number of $1$s, leading to the Lemmas~\ref{lem:2nd-lower-part-1}--\ref{lem:2nd-lower-part-3} 
below. The underlying idea is that 
both the number of $2$nd-class individuals is sufficiently large and that at the same 
time, these $2$nd-class individuals were selected from an even larger set  to allow 
many $1$s to be gained at the considered position~$j$. This requires a careful analysis 
of the level where the rank-$\mu$ individual ends up.

For 
$i \in \{0, \dots, n - 1\}$, let $C_i$ denote the cardinality of level $i$, 
\ie, the number of individuals in level $i$ during an arbitrary generation of the \umda, 
and let $\cumulC{i} = \sum_{a = i}^{n - 1} C_a$. 
Let $M$ denote the index of the first level from the top such that the number 
of sampled individuals is greater than $\mu$ when including the following level, i.e.,
\[
    M \coloneqq  \max \{i \mid \cumulC{i - 1} > \mu\}.
\]

Note that $M$ can never be $0$, and only if $M = n - 1$, $C_M$ can be greater than $\mu$. 
Due to the definition of $M$, if $M \neq n - 1$, level $M - 1$ contains the individual 
of rank $\mu + 1$, so level $M-1$ contains the cut where the best~$\mu$ out of $\lambda$ offspring 
are selected. Individuals in levels at least 
$M + 1$ are definitely $1$st-class individuals since they still will have rank at least~$\mu$ 
even if the  bit~$j$ sampled last turns out to be~$0$. $2$nd-class individuals, if any, have 
to come from levels $M$, $M - 1$ and $M-2$ (still in terms of the ranking before sampling bit~$j$). 
Individuals from level~$M$ may 
still be selected (but may also not) 
for the $\mu$ updating individuals 
even if bit~$j$ turns out as~$0$. Individuals from level~$M-2$ have to sample a~$1$ at bit~$j$ 
to be able to compete with the individuals from levels~$M$ and~$M-1$; still it is not sure that they 
will end up in the~$\mu$ updating individuals.

To obtain 
a pessimistic bound on the bias introduced by $2$nd-class individuals, we concentrate on 
level~$M-1$. Note that all individuals from level~$M-1$ sampling bit~$j$ as~$1$ 
will certainly be selected unless the $\mu-\cumulC{M}$ remaining slots for the $\mu$ best 
are filled up. 
We call the individuals from level 
$M-1$ \emph{$2$nd-class candidates} 
and denote their number by $D^*_{t+1}\coloneqq C_{M-1}$. By definition, $D^*_{t+1} = \cumulC{M-1}-\cumulC{M}$. 
We also introduce the notation $C^{**}_{t+1}\coloneqq \mu-\cumulC{M}$  and note that 
$C^*_{t+1}\ge C^{**}_{t+1}$ since $2$nd-class individuals also may come from levels $M$ and $M-2$, 
in addition to level~$M-1$. Our definition of~$D^*_{t+1}$ only covers the candidates for $2$nd-class individuals 
that come from level~$M-1$; so the candidates from levels~$M-2$ and~$M$ are not part of our notation.

In the following, we often drop the index $t+1$ from $C^*_{t+1}$, $C^{**}_{t+1}$, and $D^*_{t+1}$ if 
there is no risk of confusion.

\subsubsection*{Illustration}

Figure~\ref{fig:first-and-second} illustrates the concepts we have introduced so far. On the 
left-hand side, $\lambda=14$ individuals are ranked with respect to their \OneMax-value, ignoring bit~$j$. 
Level $M$ is the last level from the top such that at most~$\mu$ individuals are sampled in level~$M$ 
or above. The individuals from level $M+1$ and above will be selected for sure even if bit~$j$ turns 
out as~$0$ and are therefore $1$st-class individuals. 
Level~$M-1$ contains the individual of rank~$\mu$ and possibly further individuals. In general, 
if there are $D^*$ individuals in level~$M-1$ and $C_{\ge M}$ in higher levels, then these $D^*$ 
individuals are  the $2$nd-class candidates. After finally bit~$j$ has been sampled, selection 
will take the best $C^{**} = \mu-C_{\ge M}$ out of these $D^*$ candidates. They become 
$2$nd-class individuals. Recall that $C^*\ge C^{**}$ as the latter lacks possible $2$nd-class individuals 
from levels $M-2$ and $M$.

In the following, the crucial idea is to show 
  that $D^*$ is expected to be larger than $C^{**}$. 
That is, we expect to  
have more $2$nd-class candidates (in level~$M-1$) than can actually be selected 
as $2$nd-class individuals. This is dealt with in Lemma~\ref{lem:2nd-lower-part-1} below. Roughly speaking, it shows 
that the number of $2$nd-class individuals is stochastically as least as large as if it was sampled 
from a binomial distribution with parameters $\Theta(\mu)$ and $\Theta(1/\sigma_t)$, 
where $\sigma_t^2\coloneqq \sum_{i=1}^n p_{t,i}(1-p_{t,i})$ is the sampling variance of the \umda. This result can be 
interpreted as follows. It is well known that the Poisson binomial distribution with vector $\bm{p}_t$ has a mode of $O(1/\sigma_t)$ 
\cite{BaillonModePoissonBinomial}. 
Hence, if we just look at 
the number of individuals that has a certain number~$k$ of $1$s at some position, 
then this is determined by a binomial distribution 
with parameters $\lambda$ and $O(1/\sigma_t)$. Lemmas~\ref{lem:2nd-lower-part-1} and 
Lemma~\ref{lem:2nd-lower-part-1b} together show
 that essentially 
the same holds even if we consider the individuals from $C^*$, \ie, specific individuals 
from level $M-1$, each of which is 
drawn from the complicated conditional distribution induced by the definition of level~$M$. Also, it establishes 
a similar probabilistic bound for $D^*-C^{**}$, the difference between the number of $2$nd-class individuals from level~$M-1$ and 
the number of $2$nd-class candidates, since this difference  
is responsible for the selection bias. By definition, it always holds that $D^*-C^{**}\ge 1$.

For our analysis, knowledge of the sheer number of $2$nd-class individuals and 
candidates is not yet sufficient. Therefore, 
 afterwards Lemma~\ref{lem:2nd-lower-part-2} deals with the number of $1$s sampled in the $2$nd-class individuals 
and candidates. This result is then finally used to obtain Lemma~\ref{lem:2nd-lower-part-3}, which quantifies 
the bias due to selecting $2$nd-class individuals in a drift statement. More precisely, the expected value of $X_{t+1}$, the 
number of $1$s at position~$j$ in the $\mu$ selected individuals 
 at time~$t+1$,  is bounded from below depending on $X_t$. 
This statement is also formulated with respect to the expected success probabilities $\E(p_{t+1}\mid p_t)$  
in the lemma.

\begin{figure}[htb]
\begin{tikzpicture}
\tikzset{indiv/.style={fill=black,circle,inner sep=0pt, minimum size=1.2mm},
indiv2/.style={fill=black,circle,inner sep=0pt, minimum size=2mm},
indiv3/.style={fill=black!50,circle,inner sep=0pt, minimum size=2mm},
num/.style={text width=1cm}
},
\foreach \x in {0,1,4,5,6,9,10}
    {   \pgfmathtruncatemacro{\nodelabel}{\x};
        \node (C\nodelabel) at (0,\x*0.5) {};
				\draw[fill=white!90!black] (C\nodelabel) rectangle +(1,0.5) node[inner sep=0pt] (D\nodelabel) {};
    }
		\node (D3) at (3,3.5) {};
		
\draw[densely dotted,thick] ($(C1.north)+(0.5,0.5)$)  -- ($(C4.south)+(0.5,0)$);
\draw[densely dotted,thick] ($(C6.north)+(0.5,0.5)$)  -- ($(C9.south)+(0.5,0)$);
\coordinate (Off) at (0.65,-0.25);
\node[num] at ($(D10)+(Off)$) {$n-1$};
\node[num] at ($(D9)+(Off)$) {$n-2$};
\node[num] at ($(D6)+(Off)$) {$M$};
\node[num] at ($(D5)+(Off)$) {$M-1$};
\node[num] at ($(D4)+(Off)$) {$M-2$};
\node[num] at ($(D1)+(Off)$) {$1$};
\node[num] at ($(D0)+(Off)$) {$0$};
\node[indiv] at ($(C9)+(0.5,0.25)$) {};
\node[indiv] at ($(C6)+(0.2,0.25)$) {};
\node[indiv] at ($(C6)+(0.4,0.25)$) {};
\node[indiv] at ($(C6)+(0.6,0.25)$) {};
\node[indiv] at ($(C6)+(0.8,0.25)$) {};

\foreach\x in {0,...,5}
{
\node[indiv] at ($(C5)+(\x*0.15+0.12,0.25)$) {};
}
\node[indiv] at ($(C4)+(0.3,0.25)$) {};
\node[indiv] at ($(C4)+(0.5,0.25)$) {};
\node[indiv] at ($(C4)+(0.7,0.25)$) {};
\draw [decorate,decoration={brace,amplitude=7pt}]
($(D10)+(1.2,0)$) -- ($(D6.north)+(1.2,0)$) node [black,midway,xshift=1cm] {\parbox{1.3cm}{all\\$1$st class}}; 
\draw [decorate,decoration={brace,amplitude=7pt}]
($(D6.south)+(1.2,0)$) -- ($(D4.north)+(1.2,-0.5)$) node [black,midway,xshift=1cm] {\parbox{1.3cm}{$2$nd class\\from here}};

\draw [decorate,decoration={brace,amplitude=7pt}]
($(C6.north)+(-0.1,0)$) -- ($(C10)+(-0.1,0.5)$) node [black,midway,xshift=-0.5cm] {\parbox{0.5cm}{$\le\mu$\\[-0.3mm]ind.}}; 

\draw [decorate,decoration={brace,amplitude=7pt}]
($(C5.north)+(-0.9,0)$) -- ($(C10)+(-0.9,0.5)$) node [black,pos=0.5,xshift=-0.55cm] {\parbox{0.5cm}{$>\mu$\\[-0.3mm]ind.}}; 


\draw[red] ($(C5)+(-0.08,-0.08)$) rectangle ($(D5)+(0.08,0.08)$) node[outer sep=0pt, inner sep=0pt] (E1){};
\draw[red] ($(E1)+(0,-0.25)$) -- +(4.5,1.2) node (E2) {};

\path[fill=white!90!black] ($(E2)+(0,-0.6)$) node(E3) {} rectangle +(3,1.7);
\draw[fill=white!90!black] ($(E2)+(0,-0.5)$) node(E3) {} rectangle +(3,1.5);

\foreach\x in {0,1}
{
\node[indiv2] at ($(E3)+(\x*0.45+0.4,0.75)$) {};
}
\foreach\x in {2,...,5}
{
\node[indiv3] at ($(E3)+(\x*0.45+0.4,0.75)$) {};
}

\draw [decorate,decoration={brace,mirror}] ($(E3)+(0.25,0.6)$) 
-- ($(E3)+(1.0,0.6)$) node [midway,below=1mm] {$C^{**}$};

\draw [decorate,decoration={brace,mirror}] ($(E3)+(1.25,0.6)$) 
-- ($(E3)+(2.75,0.6)$) node [midway,below=1mm] {$D^*-C^{**}$};

\node at ($(E3)+(1.6,-1)$) {\parbox{3.2cm}{\raggedright\tiny After sampling bit~$j$, the $C^{**}$  best of out $D^*$ 
individuals from level $M-1$ are chosen.}};

\end{tikzpicture}
\caption{Illustration of the ranking of the individuals after sampling $n-1$ bits for $\lambda=14$ and $\mu=7$. Finally,  
$C^{**}=2$ individuals out of $D^*=6$ from level $M-1$  will definitely be selected. Some  individuals from levels~$M$ and~$M-2$ 
may also be $2$nd-class, in which case $C^*  > C^{**}$ holds.}
\label{fig:first-and-second}
\end{figure}
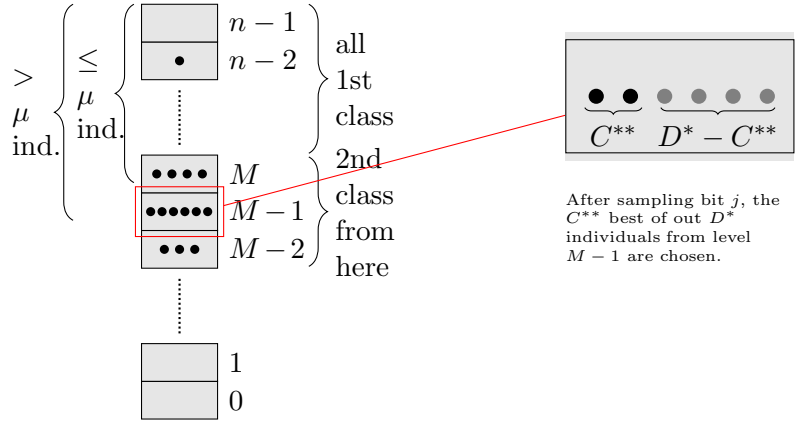

We are now ready to state the first of the above-mentioned four lemmas. It shows that the number of $2$nd-class 
individuals from level~$M-1$ follows a binomial distribution $\Bin(\mu,q)$, the second parameter of which will be 
analyzed in the subsequent Lemma~\ref{lem:2nd-lower-part-1b}. In addition, the lemma establishes a similar 
result for $D^*_{t+1} - C^{**}_{t+1}$, the overhang in $2$nd-class candidates \wrt the number of 
$2$nd-class individuals that can come from level~$M-1$. Again Lemma~\ref{lem:2nd-lower-part-1b} will 
analyze the second parameter of the respective distribution.

\begin{lemma}
\label{lem:2nd-lower-part-1}
For all $t\ge 0$, 
\begin{enumerate}
\item $C^*_{t+1}\succeq C^{**}_{t+1}$ and $C^{**}_{t+1}\sim \Bin(\mu,q)$ for some random $q\le 1$. 
\item $D^*_{t+1} - C^{**}_{t+1}\sim 1 + \Bin(\lambda-\mu-1,q')$ for some random $q'\le 1$. 
\end{enumerate}
\end{lemma}

\begin{proof}
We first prove the first statement in a detailed manner and then show that the second one can be proven similarly. 
Hence, we now concentrate on the distribution of $C^{**} = C^{**}_{t+1} =   
 \mu-\cumulC{M}$, which, as outlined above, is a lower bound 
on $C^*_{t+1}$. 
	To this end, we carefully investigate and then reformulate 
		the stochastic process generating the $\lambda$ individuals (before selection), 
	 restricted to 
		 $n-1$ bits. Each individual is 
		sampled by a Poisson binomial distribution for a vector of 
		probabilities $\bm{p}_t'=(p_{t,1}',\dots,p_{t,n-1}')$ 
		obtained from the frequency vector  of the \umda by leaving the entry belonging to bit~$j$ out 
		(\ie, $\bm{p}_t'=(p_{t,1},\dots,p_{t,j-1},p_{t,j+1},\dots,p_{t,n})$). 
		 Counting its 
			number of $1$s, each of the $\lambda$ individuals then falls into some 
			level~$i$, where $0\le i\le n-1$, with 
			some probability~$q_i$ depending on the vector $\bm{p}'_t$. Since 
			the individuals are created independently, the number 
			of individuals in level $i$ is binomially distributed with parameters $\lambda$ and~$q_i$. 
			
			Next, we take an alternative view on the process of putting individuals into levels, using 
			the principle of deferred decisions. We imagine that the 
			process first samples all individuals in level~$0$ (through $\lambda$ trials, all of which 
			either hit the level or not), then (using the trials which did not hit level~$0$) 
			all individuals in level~$1$, \dots, up  
			to level~$n-1$.
            
      The number of individuals~$C_{0}$
			in level~$0$ is still binomially distributed with parameters $\lambda$ 
			and~$q_{0}$. However, after all individuals in level~$0$  have 
			been sampled, the distribution changes. 
			We have $\lambda-C_{0}$ trials left, each of which can hit one of the 
			levels $1$ to $n-1$. In particular, such 
			a trial will hit level~$1$ with probability $q_{1}/(1-q_{0})$, 
			by the definition of conditional probability since level~$0$ is excluded. 
			This holds independently for all 
			of the $\lambda-C_{0}$ trials so that $C_{1}$ follows a binomial distribution with parameters 
			$\lambda-C_{0}$ and $q_{1}/(1-q_{0})$. Inductively, 
			also all $C_{i}$ for $i>1$ are binomially distributed; 
			\eg, $C_{n-1}$ is distributed with parameters $\lambda-C_{n-2}-\dots-C_{0}$ and~$1$. 
			Note that this model 
			of the sampling process can also be applied for any other permutation of the levels; we will make use 
			of this fact.
    
    Recall that our aim is to analyze $C^{**}$. 
		Formally, by applying the law of total probability, 
		its distribution looks as follows for $k \in \{0, \dots, \lambda\}$:
    \begin{equation}
    \begin{split}
    \Pr(&C^{**} \ge k) 
		= \sum_{i = 1}^{n -1} \Pr(M=i) \cdot \Pr(\mu-\cumulC{i}\ge k\mid M=i)\ .
    \end{split}
    \label{eq:cutLevelProbability}
		\end{equation}

		To prove the first item of the lemma, 
		it is now 
		sufficient to bound  $\Pr(\mu-\cumulC{i}\ge k\mid M=i)$ by the distribution function belonging to a binomial 
		distribution for all $i\in\{1,\dots,n-1\}$ (recalling that $M=0$ is impossible). 
		
		We reformulate the underlying event  appropriately. 
		Here we note that \[\{\mu-\cumulC{i}\ge k\}\cap \{M=i\}\] is equivalent to 
		\[
		\{C_{\le i-1}\ge \lambda-\mu+k\}\cap \{M=i\},
		\]
		where $C_{\le i}=\sum_{j = 0}^{i}C_j$, 
		and, using the definition of $M$, this is also equivalent to \[
		\{C_{\le i-1} \ge \lambda-\mu+k\} \cap 
		\{C_{\le i-2}< \lambda-\mu\}.\]
		
		We now use the above-mentioned  view on the stochastic process 
		 and assume that levels $0$ to $i-2$ have been sampled and a number of experiments 
		in a binomial distribution is carried out to determine the individuals 
		from level~$i-1$. Hence, considering the outcome of $C_{\le i-2}$ and using that $C_{i-1} = C_{\le i-1} - C_{\le i-2}$, 
		our event is equivalent to that the event
		\[
		E^*\coloneqq \{C_{i-1} \ge (\lambda-\mu-a) + k\big\} \cap \{C_{\le i-2} = a\} 	\]
		happens for some $a<\lambda-\mu$.
		Recall from our model that $C_{i-1}$ follows a binomial distribution with 
		$\lambda-a$ trials and with a certain success probability~$s$. The number of 
		trials left after having sampled levels $0,\dots,i-2$  is at least~$\mu$ since $a<\lambda-\mu$. 
		The probability of $E^*$ is determined by 
		conditioning  on that $C_{\le i-2}=a$, \ie, $a$ samples have fallen into levels 
		$0,\dots,i-2$ and that afterwards $i-1$ has already 
		been hit by $\lambda-\mu-a$ samples. Then $\mu$ trials are left 
		that still may sample within $C_{i-1}$. 
		 Altogether, we proven that 
		$C_{i-1}$, conditioning on~$M=i$, follows a binomial distribution with parameters $\mu$ and $q$, where 
		the value of~$q$ depends on the random~$M$. This proves the first item of the lemma.

		

		We now use a dual line of argumentation to  prove the second item 
		of the lemma. While the item is concerned with $D^*_{t+1}-C^{**}_{t+1}$, it 
		is more convenient to analyze $C_{\ge M-1}$ and then exploit that 
		\begin{equation}
		D^*_{t+1}-C^{**}_{t+1} = (C_{\ge M-1}-C_{\ge M}) - (\mu-C_{\ge M}) = C_{\ge M-1}-\mu,
		\label{eq:lemma8-help}
		\end{equation}
		which follows directly from the definition of $C^{**}_{t+1}$ and $D^*_{t+1}$.
		
		We claim that 
		\begin{equation}
		C_{\ge M-1} \sim \mu + 1 + \Bin(\lambda-\mu-1,q'),
		\label{eq:lemma8-part3}
		\end{equation}
		for some probability~$q'$ depending on the outcome of~$M$.
		To show this, 
		we take the same view on the stochastic process as above but imagine 
		now that the levels are sampled in the order from $n-1$ down to~$0$. Conditioning 
		on that levels $n-1,\dots,M$ have been sampled, 
		there are at least $\lambda-\mu$ trials are left 
		to populate level~$M-1$ 
		 since  by definition less than $\mu$ samples 
		fall into levels $n-1,\dots,M$. However, by definition of~$M$, at least 
		$\mu+1-C_{\ge M}$ of these trials must fall into level~$M-1$. Afterwards, there are 
		$\lambda-\mu-1$ trials left, each of which may hit level~$M-1$ or not. This proves
		\eqref{eq:lemma8-part3}.
\qed\end{proof}

As announced, the purpose of the following lemma is to analyze the second parameters 
of the binomial distributions that appear in Lemma~\ref{lem:2nd-lower-part-1}. Roughly speaking, 
up to an exponentially small failure probability, we obtain $\Omega(1/\sigma_t)$ as success probability of 
the binomial distribution.

\begin{lemma}
\label{lem:2nd-lower-part-1b}
Let $\sigma_t^2\coloneqq \sum_{i=1}^n p_{t,i}(1-p_{t,i})$ be the sampling variance of the \umda. 
Consider $C^{**}_{t+1}\sim \Bin(\mu,q)$  and $D^*_{t+1} - C^{**}_{t+1}\sim 1 + \Bin(\lambda-\mu-1,q')$ 
 as defined in Lemma~\ref{lem:2nd-lower-part-1}.  
There  is an event~$E^*$ with $\Prob(E^*)=1-e^{-\Omega(\mu)}$ 
such that for all $t\ge 0$ the following holds:
\begin{enumerate}
\item Conditioned on $E^*$, it holds that  $q=\Omega(1/\sigma_t)$. 
Hence $\E(C^{**}_{t+1}\mid \sigma_t)=\Omega(\mu/\sigma_t)$.
\item 
Conditioned on $E^*$, it holds that
  $q'=\Omega(1/\sigma_t)$. Hence $\E(D^*_{t+1}-C^{**}_{t+1}\mid \sigma_t) = 1+\Omega((\lambda-\mu-1)/\sigma_t)$.
\end{enumerate}
\end{lemma}

\begin{proof}
We recall the first statement from Lemma~\ref{lem:2nd-lower-part-1} and the stochastic model 
used in its proof. Let $X$ be the number of $1$s at the considered position~$j$ in a single individual sampled 
		in the process of creating the~$\lambda$ offspring 
		(without conditioning on certain levels being hit). The aim is to 
		derive bounds on~$q$  using~$X$. By our stochastic model, $q$ 
		denotes the probability to sample an individual with $M-1$ $1$s, 
		given that it cannot have less than $M-1$ $1$s. By omitting this condition, we clearly do 
		not increase the probability. Hence, 
		we pessimistically assume that $q=\Prob(X=i-1)$, given $M=i$. The latter probability heavily depends 
		on~$M$. We will now concentrate on the values of~$i$ where $\Pr(M=i)$ is not 
		too small.
		
		The random variable~$X$ follows a Poisson binomial distribution with vector $\bm{p}'_t$ as defined in the proof of 
		Lemma~\ref{lem:2nd-lower-part-1}. Clearly, the variance of this distribution, call it $\tilde{\sigma}_t^2$, is smaller than $\sigma_t^2$ since 
		bit~$j$ is left out. Still, since $\sigma_t^2\ge n\cdot (1/n)(1-1/n)=1-1/n$ due to the borders on the frequencies, 
		we obtain $\tilde{\sigma}_t^2\ge 1-1/n-1/4$ and $\tilde{\sigma}_t^2 = \Theta(\sigma_t^2)$. 
		
		We define 
		\[L\coloneqq \min\left\{i\mid \Prob(X\le i)\ge \frac{1}{2+2\beta}\right\}\]
		and 
		\[U\coloneqq \max\left\{i\mid \Prob(X\ge i)\ge \frac{\beta}{2+2\beta}\right\},\]
		where $\beta$ is still the constant from our assumption $\lambda=(1+\beta)\mu$. 
		By 
		Chernoff bounds, both the 
    number of individuals sampled above~$U$ is less 
		than $\frac{1}{1+\beta}\lambda=\mu$  
		and 
    the number of individuals sampled below~$L$ 
		is  less than $\frac{\beta}{1+\beta}\lambda=\lambda-\mu$ 
		with probability $1-e^{-\Omega(\lambda)}$. 
		 Then the $\mu$-th ranked individual will be within 
    $Z\coloneqq [L,U]$ with probability at least $1-e^{-\Omega(\mu)}$, which means that 
		\begin{equation}
		\Prob(M\in Z)=1-e^{-\Omega(\mu)}. 
		\label{eq:lemma8-minz}
		\end{equation}
		Note that $M\in Z$ is the event $E^*$ mentioned in the statement of the lemma.

		We  now assume $M\in Z$ and 
		apply Lemma~\ref{lem:at-least-1-over-sigma}, using $\ell\coloneqq 1/(2+2\beta)$ and 
		$u\coloneqq \beta/(2+2\beta)$, in accordance with the above definition of $L$ and~$U$. 
		Hence, every level in~$Z$, in particular level~$M-1$, is 
		hit with probability $\Omega(\min\{1,1/\tilde{\sigma}_t\}) = \Omega(1/\sigma_t)$. 
		Hence with probability~$1-e^{-\Omega(\mu)}$ 
		we have that $q=\Omega(1/\sigma_t)$ and therefore 
		$C^{**}_{t+1}  \sim \Bin(\mu,\Omega(1/\sigma_t))$. 
		Using the properties 
		of the binomial distribution and the law of total probability, 
		we obtain $\E(C^{**}) = \Omega(\mu/\sigma_t)$, which proves the first item 
		of the lemma. 
		
		The second item is proven similarly. 
		 We recall from 
		Lemma~\ref{lem:2nd-lower-part-1} that 
		$D^*_{t+1}-C^{**}_{t+1} = 1+ \Bin(\lambda-\mu-1,q')$ 
		where $q'$ is the probability of hitting level $M-1$, assuming levels~$n-1,\dots,M$ 
		have been sampled. 
    Hence 	
		with probability $1-e^{-\Omega(\mu)}$ according to \eqref{eq:lemma8-minz}, 
		we have $q'=\Omega(1/\sigma_t)$.
		Altogether, we obtain 
		$\E(D^*_{t+1}-C^{**}_{t+1})=1+\Omega((\lambda-\mu-1)/\sigma_t)$, which concludes the  proof of the 
		second item of the lemma.
\qed\end{proof}

As mentioned above, we now know much about the distribution of 
the number of $2$nd-class individuals and candidates. The next step is 
to bound the number of $1$s sampled at position~$j$ in these individuals.

\begin{lemma}
\label{lem:2nd-lower-part-2}
For all $t\ge 0$
\[
X_{t+1}\succeq \Bin(\mu-C^{**}_{t+1},X_t/\mu)+\min\{C^{**}_{t+1},\Bin(D^*_{t+1},X_t/\mu)\}.
\]
\end{lemma}

\begin{proof}
		We essentially show that 
		the expected overhang in $2$nd-class candidates from level~$M-1$ compared to 
		$2$nd-class individuals from this level 
		allows a bias of the frequency towards higher values, as detailed in the following. We recall that we ignore 
		$2$nd-class individuals stemming from levels~$M-2$ and~$M$. These might introduce a bias that is only 
		larger, which is why the statement of the lemma only establishes a lower bound on~$X_{t+1}$.

    In each of the $D^*_{t+1}$ $2$nd-class candidates from level~$M-1$, 
		bit~$j$ is sampled as~$1$ with probability $X_t/\mu$. Only 
		a subset of the candidates, namely the $C^{**}_{t+1}$ $2$nd-class individuals from this level,
		is selected for the best $\mu$ offspring determining 
		the next frequency. As observed above in Section~\ref{sec:first-and-second-class}, the number of $1$s at position~$j$ in the $1$st-class individuals 
		is binomially distributed with parameters $\mu-C^*_{t+1}$ and $X_t/\mu$. We have $C_{t+1}^{*}\ge C^{**}_{t+1}$ and recall  
		that the distribution of $X_{t+1}$ becomes stochastically smallest when equality holds. 
    Hence, we obtain for the number of $1$s (at position~$j$) in the $\mu$ selected offspring that 
		\[
		X_{t+1}\succeq \Bin(\mu-C^{**}_{t+1},X_t/\mu)+\min\{C^{**}_{t+1},\Bin(D^*_{t+1},X_t/\mu)\},
		\]
		which is what we wanted to show.
		\qed\end{proof}

		Finally, based on the preceding two lemmas, we can quantify the bias 
		of the frequencies due to selection in a simple drift statement. The following lemma is crucially used in the 
		drift analyses that prove Theorems~\ref{theo:upper-above-phase} and~\ref{theo:upper-below-phase}.

\begin{lemma}
\label{lem:2nd-lower-part-3}
Let $\mu=\omega(1)$. Then 
for all $t\ge 0$, 
 \[\E(X_{t+1}\mid X_t,\sigma_t) = X_t + \Omega\bigl((\mu/\sigma_t)(X_t/\mu)(1-X_t/\mu)\bigr	).\]		

If $p_t\le 1-c/n$, where $c>0$ is a sufficiently large constant, then 
 \[\E(p_{t+1}\mid p_t,\sigma_t) = p_t + \Omega(p_t(1-p_t)/\sigma_t).\]
\end{lemma}

\begin{proof}		
		We start with  
		the bound 
		\begin{equation}
		X_{t+1}\succeq \Bin(\mu-C,X_t/\mu)+\min\{C,\Bin(D,X_t/\mu)\}
		\label{eq:lem:2nd-lower-part-3-1}
		\end{equation}
		from Lemma~\ref{lem:2nd-lower-part-2}, where we write  $C\coloneqq C^{**}_{t+1}$ and 
		$D\coloneqq D^*_{t+1}$ for notational convenience. We will estimate the expected value 
		of~$X_{t+1}$ based on this stochastic lower bound. By Lemma~\ref{lem:concave-expec}, 
		the expected value of the minimum is at least as large as the minimum of \[
		C\frac{X_t}{\mu}+\frac{1}{4} (D-C)\frac{X_t}{\mu}\left(1-\frac{X_t}{\mu}\right)\] 
		and 
		\[
		C\frac{X_t}{\mu}+ \frac{1}{4} C\frac{X_t}{\mu}\left(1-\frac{X_t}{\mu}\right),
		\]
		where $C$ and $D$ are still random. 
		
	  Taking the expected value in \eqref{eq:lem:2nd-lower-part-3-1}, we obtain 
		\begin{align}
		 \E(X_{t+1} \mid X_t,\sigma_t,C,D)  & \ge 
		\E(\Bin(\mu-C,X_t/\mu)\mid X_t, \sigma_t, C,D)  
		 + 
		C \frac{X_t}{\mu}  \notag\\ 
		& \qquad\qquad + \frac{1}{4} \frac{X_t}{\mu}\left(1-\frac{X_t}{\mu}\right) \min\{C,D-C\}\notag\\
		& \ge X_t + \frac{1}{4} \frac{X_t}{\mu}\left(1-\frac{X_t}{\mu}\right) \min\{C,D-C\},
		\label{eq:lem:2nd-lower-part-3-2}
		\end{align}
		where the last inequality computed the expected value of the binomial distribution. We also note that $C$ and $D$ are independent 
		of $X_t$ as $X_t$ counts the number of ones in bit~$j$, which is not used to determine $C$ and $D$.
		
		Recall that the overall aim is to bound $\E(X_{t+1} \mid X_t,\sigma_t) = \E(\E(X_{t+1} \mid X_t,\sigma_t,C,D)\mid X_t,\sigma_t)$ from below. 
		Hence, inspecting the last bound from \eqref{eq:lem:2nd-lower-part-3-2}, 
		we are left with the task to prove a lower bound on 
		\[\E(\min\{C,D-C\}\mid X_t,\sigma_t) = \E(\min\{C,D-C\}\mid \sigma_t),\]
		using that $C$ and~$D-C$ are independent of~$X_t$. 
		We recall from Lemmas~\ref{lem:2nd-lower-part-1} and~\ref{lem:2nd-lower-part-1b} that $C\succeq \Bin(\mu,q)$ and 
		$D-C\succeq 1+\Bin(\lambda-\mu-1,q')$, where $q=\Omega(1/\sigma_t)$ and $q'=\Omega(1/\sigma_t)$ with probability 
		$1-e^{-\Omega(\mu)}$. 
		Also, Lemma~\ref{lem:2nd-lower-part-1b} yields $E(C\mid \sigma_t)=\Omega(\mu/\sigma_t)$ and 
		$E(D-C\mid \sigma_t)\ge 1+\Omega((\lambda-\mu-1)/\sigma_t) = 1 + \Omega(\mu/\sigma_t) $, where the last equality used our assumption that 
		$\lambda=(1+\Theta(1))\mu$.
		
		We distinguish between two cases. If $\mu/\sigma_t \le \kappa$ for an arbitrary constant~$\kappa>0$ (chosen later) 
		then 
		$1/\sigma_t=O(1/\mu)=o(1)$ by our assumption on~$\mu$. 
		Working with the lower bound~$1$ on~$D-C$, we get 
		\begin{align*}
		\E(\min\{C,D-C\}\mid  \sigma_t) & \ge 
		\E(C\cdot \indic{C\le 1} \mid \sigma_t) \\
		& \ge \Prob(C=1) = \binom{\mu}{1} \Omega\bigl((1/\sigma_t)(1-o(1)\bigr) 
		 = \Omega(\mu/\sigma_t).
		\end{align*}
		If $\mu/\sigma_t \ge \kappa$ and $\kappa$ is chosen sufficiently large but constant then Chernoff bounds yield that 
		each of the events 
		\[
		C \ge \frac{E(C)}{2}
		\]
		and 
		\[
		D\ge \frac{E(D-C)}{2}
		\]
		fail to happen 
		with probability at most~$1/4$, so by a union bound both events 
		happen simultaneously with probability at least~$1/2$. This proves 
		\[
		\E(\min\{C,D-C\}\mid \sigma_t) = \Omega(\mu/\sigma_t)
		\]
		also in this case.
		Plugging this back into \eqref{eq:lem:2nd-lower-part-3-2}, we obtain 
		\[
		\E(X_{t+1}\mid X_t) = X_t +  \Omega\bigl((\mu/\sigma_t)(X_t/\mu) (1-X_t/\mu)\bigr),
		\]
		which proves the statement on $\E(X_{t+1}\mid X_t,\sigma_t)$ from this lemma. 
		
		To conclude on the expected value of $p_{t+1}$, we recall from Algorithm~\ref{alg:UMDA} that 
		$p_{t+1}\coloneqq \mathrm{cap}_{1/n}^{1-1/n} (X_{t+1}/\mu)$. Using our assumption $p_t\le 1-c/n$ we get 
		$1-X_{t}/\mu \ge c/n  $. Hence, 
		as hitting the upper border changes the frequency by only at most~$1/n$ and the lower border 
		can be ignored here, we also 
		obtain  
		\[
		\E(p_{t+1}\mid p_t) \ge p_t +  \Omega((1/\sigma_t)p_t (1-p_t))
		\]
		if $c$ is large enough to balance the implicit constant in the $\Omega$. 
\qed\end{proof}

We note that parts of the analyses behind Lemmas~\ref{lem:2nd-lower-part-1}--\ref{lem:2nd-lower-part-3} are 
inspired by \cite{KrejcaWittFOGA2017}; in particular, 
the modeling of the stochastic process and the definition of~$M$ follow 
that paper closely. However, 
as \cite{KrejcaWittFOGA2017} is concerned with lower bounds on the running time, it  bounds the number 
of $2$nd-class individuals from above and needs a very different argumentation in the core 
of its proofs.

\section{Above the Phase Transition}
\label{sec:upperBound}

We now prove our main result for the case of large~$\lambda$. It implies an $O(n\log n)$ running time 
behavior if $\mu = c\sqrt{n}\log n$.

\begin{theorem}
\label{theo:upper-above-phase}
Let $\lambda=(1+\beta)\mu$ for an arbitrary constant~$\beta>0$, let $\mu\ge c\sqrt{n}\log n$ for some sufficiently large 
constant~$c>0$ as well as $\mu=n^{O(1)}$. Then with probability $\Omega(1)$, the optimization time of both \umda and \umdastar 
on $\OneMax$ is bounded from above 
by $O(\lambda \sqrt{n})$. For \umda, also the expected optimization time is bounded in this way.
\end{theorem}

The proof of Theorem~\ref{theo:upper-above-phase} follows a well-known approach that 
is similar to techniques 
 partially independently proposed in  
several previous analyses of EDAs and of ant 
colony optimizers  \cite{Neumann2010a, FriedrichEtAlISAAC15, SudholtWitt2016}. 
Here we show that the approach also works for the \umda. 
Roughly, a drift analysis is performed with respect to the sum of frequencies.
In Lemma~\ref{lem:2nd-lower-part-3}, we have already established a drift of 
frequencies towards higher values. Still, there are random fluctuations 
(referred to as \emph{genetic drift} in \cite{SudholtWitt2016}) of frequencies that may lead to undesired 
decreases towards~$0$. The proof of Theorem~\ref{theo:upper-above-phase}  uses that 
under the condition on~$\mu$, typically all frequencies stay sufficiently far away from the lower border;  
more precisely, no frequency 
 drops below $1/4$.  Then 
the drift is especially beneficial.

The following lemma formally shows that, if $\mu$ is not too small, the positive drift along with 
the fine-grained scale implies  
that the frequencies will generally move to higher values and are unlikely to decrease by a large distance. 
Using the lemma, we will obtain a failure probability of $O(n^{-cc'})$ within $n^{cc'}$ generations, which 
can subsume any polynomial number of steps by choosing $c$ large enough.

\begin{lemma}
\label{lem:drift-theorem-for-probabilities}
Consider an arbitrary bit and let~$p_t$ be its frequency at time~$t$. Suppose that  
$\mu\ge c\sqrt{n}\log n$ for a sufficiently large constant~$c>0$. For $T\coloneqq \min\{t\ge 0\mid p_t\le 1/4\}$ 
it then holds  that 
$\Prob(T\le e^{cc'\log n}) = O(e^{-cc'\log n})$, where $c'$ is another positive constant.
\end{lemma}

\begin{proof}
The aim is to apply Theorem~\ref{theo:negative-drift-scaling-2017}. 
We consider the frequency $p_t\coloneqq p_{t,i}$ associated with the considered bit~$i$  
and its distance $X_t\coloneqq\mu p_{t}$ from the lower border. By initialization 
of the UMDA, $X_0=\mu/2$. Note 
that $X_t$ for $t>0$ is a process on $\{\mu/n,1,2,\dots,\mu-1,\mu(1-1/n)\}$. 

In the notation of the drift theorem, we set $[a,b]\coloneqq [\mu/4,\mu/2]$, hence 
$\ell=\mu/4$. Next we establish the three conditions. First, 
we observe that 
$\E(X_{t+1}-X_{t}\mid X_t)=\Omega(X_t(1-X_{t}/\mu)/\sqrt{n})$ (Lemma~\ref{lem:2nd-lower-part-3} 
along with the trivial bound $\sigma_t=O(\sqrt{n})$) 
for $X_t\in\{1,\dots,\mu\}$. The bound is $\Omega(\mu/\sqrt{n})$ for $X_t\in[a,b]$. Since 
$\mu\ge c\sqrt{n}\log n$ by assumption, we will 
set $\epsilon'\coloneqq  c c_1 \log n$ 
for some constant~$c_1>0$. Hereinafter, we will omit the conditions $X_t;a<X_t<b$ from the expected values. 
To establish the first condition of the 
drift theorem, 
we need to show that the expected value is ``typical''; formally,
we will find a not too large $\kappa$ such that 
\begin{equation}
\E\bigl((X_{t+1}-X_{t})\cdot\indic{X_{t+1}-X_t\le \kappa \epsilon'}\bigr)\ge \frac{\epsilon'}{2}.
\label{eq:drift-for-prob-eq}
\end{equation}
Then 
the first condition is established with $\epsilon\coloneqq \epsilon'/2 = c c_1 (\log n)/2$.

In the following, we show the claim that 
\[
\E((X_{t+1}-X_{t})\cdot\indic{X_{t+1}-X_t> \kappa'\epsilon}) \le \frac{\epsilon'}{2}\] if 
$\kappa$ is chosen sufficiently large. Obviously, this implies 
\eqref{eq:drift-for-prob-eq}. 
We recall 
from Lemma~\ref{lem:2nd-lower-part-2} 
that $X_{t+1}$ is stochastically at least as large as the sum of two random variables 
$Z_1\sim \Bin(\mu-C^{**},X_t/\mu)$ and $Z_2\sim \min\{C^{**},\Bin(D^*,X_t/\mu)\}$ 
for some random variables $C^{**}$ and $D^*$, which are related to binomial distributions according 
to Lemma~\ref{lem:2nd-lower-part-1}.  
 Theorem~\ref{theo:negative-drift-scaling-2017} allows us
to deal with a stochastic lower bound $\Delta_t(X_{t+1}-X_t)$ on the drift. We make the drift stochastically smaller 
in several ways. First, we assume that $\mu = c\sqrt{n}\log n$ and that $\sigma_t=\Omega(\sqrt{n})$, \ie, 
the pessimistic bounds used to estimate $\epsilon'$ above hold actually with equality. This is possible 
since the lower bound on the drift derived in the analysis from Lemma~\ref{lem:2nd-lower-part-3}, using the 
insights from Lemmas~\ref{lem:2nd-lower-part-1}--\ref{lem:2nd-lower-part-2}, becomes stochastically 
smaller by decreasing the two parameters. 

Second, we assume that $X_{t+1} = Z_1+Z_2$ 
 (essentially ignoring $2$nd-class individuals coming from levels 
$M$ and~$M-2$). But now the $X_{t+1}$ obtained in this 
is stochastically also bounded from above by 
$Z_1+\Bin(D^*,X_t/\mu)$, by omitting the minimum in $Z_2$. We conclude that 
\begin{align*}
X_{t+1} & \preceq \Bin(\mu-C^{**},X_t/\mu) + \Bin(D^{**},X_t/\mu) \\
&  = \Bin(\mu,X_t/\mu) + 
\Bin(D^{*}-C^{**},X_t/\mu).
\end{align*}
By Lemma~\ref{lem:2nd-lower-part-1}, $D^*-C^{**}-1$ is binomially 
distributed with known number of successes $\lambda-\mu-1$, which is $\Theta(\sqrt{n}\log n)$ by our assumptions, 
yet random success probability. We note 
that the lower bounds on $\E(X_{t+1}\mid X_t,\sigma_t)$ from Lemma~\ref{lem:2nd-lower-part-3} stem from the case 
that the success probability is $\Omega(1/\sigma_t)$, which actually happens 
with probability $1-e^{-\Omega(\mu)}$. So, without decreasing the drift, 
we can assume that the success probability 
is actually fixed at $\Theta(1/\sigma_t)=\Theta(1/\sqrt{n})$.
These estimations  lead to the bound 
\[
X_{t+1} \preceq \Bin(\mu,X_t/\mu) + 
\Bin(\Theta(\sqrt{n}\log n),\Theta(1/\sqrt{n})).
\]
To ease notation, we write $Z_1'\sim \Bin(\mu,X_t/\mu)$ and  
$Z_2'\sim \Bin(c_2\sqrt{n}\log n,c_3/\sqrt{n})$ for some 
unknown constants~$c_2,c_3>0$ and obtain $X_{t+1} = Z_1+Z_2\preceq Z_1'+Z_2'$. 
We also recall that $X_t/\mu\in[1/2,3/4]$. 
By Chernoff bounds, $\Prob(Z_1'\ge \E(Z_1')+c_4(\log n)\sqrt{\mu})\le e^{-\Omega(c_4\log^2 n)}\le 1/\mu^2$ 
if $c_4$ is chosen large enough but constant.  
Similarly, $\Prob(Z_2'\ge \E(Z_2')+c_5\log n)\le 1/\mu^2$ if $c_5$ is a sufficiently large constant.
We now look into the event $Z_1'+Z_2' \ge \E(Z_1')+c_4(\log n)\sqrt{\mu} + \E(Z_2')+c_5\log n$. A necessary 
condition for this to happen is that at least one of the two events  $Z_1'\ge \E(Z_1')+c_4(\log n)\sqrt{\mu}$ and 
$Z_2'\ge \E(Z_2')+c_5\log n$ happens. A union bound yields that 
\[
\Prob(Z_1'+Z_2' \ge \E(Z_1'+Z_2')+c_4(\log n)\sqrt{\mu}+c_5\log n ) \le \frac{2}{\mu^2}
\]
and therefore clearly 
\[
\Prob(Z_1'+Z_2' \ge \E(Z_1'+Z_2')+c_6(\log n)\sqrt{\mu}) \le \frac{2}{\mu^2}
\]
for some constant~$c_6>0$. Since $X_{t+1}\preceq Z_1'+Z_2'$, we conclude that 
\[
\Prob(X_{t+1} \ge \E(Z_1'+Z_2')+c_6(\log n)\sqrt{\mu}) \le \frac{2}{\mu^2}.
\]

Since $X_{t+1}\le \mu$ and  $\E(Z_1'+Z_2')\le X_t+c_7\log n$ for $c_7=c_2/c_3$, we obtain   
\[
\E(X_{t+1} \cdot \indic{X_{t+1}\ge X_t + c_7\log n  +c_6(\log n)\sqrt{\mu}} ) \le \frac{2}{\mu}
\]
and, since $X_t\ge 0$, clearly also 
\[
\E\bigl((X_{t+1}-X_t) \cdot \indic{X_{t+1}-X_t \ge  (c_6+c_7)(\log n)\sqrt{\mu}}\bigr ) \le \frac{2}{\mu} \le \frac{\epsilon'}{2}.
\]
Note that $(c_6+c_7)(\log n)\sqrt{\mu}=\Theta(\sqrt{\mu}\epsilon')$. 
This proves the claim for $\kappa=\Theta(\sqrt{\mu})$ and establishes the 
first condition of the drift theorem.

To show the second condition, recall from 
 Section~\ref{sec:first-and-second-class} that $X_{t+1}$ stochastically dominates $\Bin(\mu,X_t/\mu)$. Hence, to analyze steps 
where $X_{t+1}<X_t$, we may pessimistically assume the martingale case, where  $X_{t+1}$ follows this binomial distribution, and obtain  
\[\varsigma^2\coloneqq \Var(X_{t}-X_{t+1}\mid X_t)=\mu\frac{X_t}{\mu}\left(1-\frac{X_t}{\mu}\right)\le \frac{\mu}{4},\]
so $\varsigma=O(\sqrt{\mu})$. 
Using Lemma~\ref{lem:mcdiarmid} with $d=j\varsigma$ and $b=1$, we get 
$\Prob(X_{t+1}-X_t \le -j\varsigma) \le e^{-\Omega(\min\{j^2,j\})}$. Hence, we can work with some $r=c_7\sqrt{\mu} $ 
for some sufficiently large constant~$c_7>0$ and satisfy the second condition on jumps that decrease the state. 

The third condition is 
also easily verified.
We recall that $\ell=\Theta(\mu)$, $\epsilon=\Theta(\log n)$, $r=\Theta(\sqrt{\mu})$ and $\kappa=\Theta(\sqrt{\mu})$. 
We note that $\epsilon/r^2=\Theta((\log n)/\mu)$, $1/r=\Theta(1/\sqrt{\mu})$ and 
$1/\epsilon=\Theta(1/\!\log n)$. Hence, the $\lambda$ from the drift theorem (not to be confused 
with the $\lambda$ of the UMDA) equals 
$\epsilon/(17r^2)$ and $\lambda\ell  =  \Omega(c\log n)$ by 
our assumption $\mu\ge c\sqrt{n}\log n$ 
from the lemma. Recalling that $\epsilon\ge c c_1(\log n)/2$ and  $\ell=\mu/4$, we obtain 
by choosing $c$ large enough that 
\[
\frac{\epsilon \ell}{17r^2} \ge 2\ln\left(\frac{4r^2}{17\epsilon^2} \right )
\]
since the right-hand side is at most $2\ln(\Theta(\mu/\!\log^2 n)) = \Theta(\ln n)$ due to $\mu = n^{O(1)}$. Thereby, we satisfy the third condition. 
Hence, the drift theorem implies that the first hitting time of states less than~$a$, starting 
from above $b$ is at least $e^{cc'\log n}$ with probability at least $1-e^{-cc'\log n}$, if $n$ is large enough and $c'$ 
is chosen as a sufficiently small positive constant independent of~$c$. 
\qed\end{proof}

We now ready to prove the main theorem from this section.

\begin{proofof}{Theorem~\ref{theo:upper-above-phase}}
We use a similar approach and partially also similar presentation of the ideas
 as in \cite{SudholtWitt2016}.
Following~\cite[Theorem~3]{Neumann2010a} we show that, starting with a setting where all 
frequencies are at least~$1/2$ simultaneously,
with probability~$\Omega(1)$ after $O(\sqrt{n})$ generations either
the global optimum has been found or at least one frequency has dropped below~$1/4$.
In the first case we speak of a success and in the latter   of a failure.
The expected number of generations until either a success or a failure
happens is $O(\sqrt{n})$.

With respect to \umda, we can use the success probability~$\Omega(1)$ 
to bound the expected optimization time. 
We choose a constant $\gamma > 3$. 
According to Lemma~\ref{lem:drift-theorem-for-probabilities}, the probability of a failure
in altogether $n^{\gamma}$ generations is at most $n^{-\gamma}$, provided the constant~$c$ in the condition $\mu \ge c\sqrt{n}\log n$ is large enough.
In case of a failure we wait until all frequencies
simultaneously reach values at least~$1/2$ again and then repeat the arguments from the preceding paragraph.
It is easy to show via additive drift analysis for the \umda (not the \umdastar) 
 that the expected time for one frequency 
to reach the upper border is always bounded by $O(n^{3/2})$, regardless of the initial probabilities. This holds  
since by Lemma~\ref{lem:2nd-lower-part-3} there is always an additive drift of 
$\Omega(p_{t,i}(1-p_{t,i})/\sigma_t)=\Omega(1/(n\sigma_t))=\Omega(1/n^{3/2})$. 
By standard arguments on independent phases, the expected time until \emph{all} frequencies 
have reached their upper border at least once is $O(n^{3/2} \log n)$.
Once a frequency reaches the upper border, we apply a straightforward modification of Lemma~\ref{lem:drift-theorem-for-probabilities} 
 to show that the probability of a frequency decreasing below $1/2$ in time $n^{\gamma}$ is at 
most $n^{-\gamma}$ (for large enough~$c$). The probability that there is a frequency
 for which this happens is at most $n^{-\gamma + 1}$ by the union bound. If this does not happen, all frequencies 
attain value at least $1/2$ simultaneously, and we apply our above arguments again.
As the probability of a failure is at most $n^{-\gamma+1}$, the expected number of restarts is $O(n^{-\gamma+1})$
and the expected time until all bits recover to values at least $1/2$ only leads to
an additional term of $n^{-\gamma+1} \cdot O(n^{3/2} \log n) \le  o(1)$ (as $n^{-\gamma} \le n^{-3}$) in the expectation.
We now only need to show that after $O(\sqrt{n})$ generations without failure
the probability of having found the all-ones string is $\Omega(1)$.

In the rest of this proof, we consider the potential function $\phi_t\coloneqq n-1-\sum_{i=1}^n p_{t,i}$, 
which denotes the total distance of the frequencies from the upper border $1-1/n$. For simplicity, 
for the moment we assume that no frequency is greater than $1-c/n$, where 
$c$ is the constant from Lemma~\ref{lem:2nd-lower-part-3}. Using 
Lemma~\ref{lem:2nd-lower-part-3} and the linearity of expectation, we obtain for some constant~$\gamma>0$ 
the drift 
\begin{align*}
 \E(\phi_t-\phi_{t+1}\mid \phi_t) & = \sum_{i=1}^n (p_{t+1,i}-p_{t,i}) \\
& = \sum_{i=1}^n (p_{t,i}+\gamma 
p_{t,i}(1-p_{t,i})/\sigma_t - p_{t,i}) = \gamma 
\sigma_t,
\end{align*}
since $\sum_{i=1}^n p_{t,i}(1-p_{t,i})=\sigma_t^2$. Using our assumption 
$p_{t,i}\ge 1/4$, we obtain the lower bound 
\begin{align}
 \E(\phi_t-\phi_{t+1}\mid \phi_t) & \ge \gamma \sqrt{\sum_{i=1}^n p_{t,i}(1-p_{t,i})} 
 \ge 
 \gamma\sqrt{\sum_{i=1}^n (1-p_{t,i})/4} = \gamma \frac{\sqrt{\phi_t}}{2}. 
\label{eq:upper-bound-above-drift-sqrt}
\end{align}

The preceding analysis ignored frequencies above $1-c/n$. To take these into account,  we now consider an arbitrary 
but fixed frequency being greater than $1-c/n$. We claim that in each of the $\mu$ 
selected offspring the underlying bit (say, bit~$j$)  
is set to~$0$ with probability at most~$c/n$. Clearly, each of the $\lambda$ offspring (before selection) has bit~$j$
 set to~$0$ with probability at most~$c/n$. Let us again model the sampling of offspring as a process where bit~$j$ 
is sampled last after the outcomes of the other $n-1$ bits have been determined. 
Since setting bit~$j$ to~$0$ leads to a lower $\om$-value than setting it to~$1$, 
the probability that  bit~$j$ is~$0$ in a selected offspring cannot be larger than~$c/n$.  

By linearity of expectation, an expected number of at most $\mu c/n$ out of $\mu$ selected offspring set bit~$j$ to~$0$. 
Hence, the expected next value of frequency of bit~$j$ must satisfy 
\[
\E(p_{t+1,j} \mid p_{t,j}; p_{t,j}>1-c/n) \ge \frac{(\mu-\mu c/n)\cdot 1 + (\mu c/n) \cdot 0}{\mu} = 1-\frac{c}{n},
\]
so $\E(p_{t,j} - p_{t+1,j}) \mid p_{t,j}; p_{t,j}>1-c/n) \ge c/n$. 
Again by 
linearity of expectation, the frequencies greater than $1-1/c$ contribute to the 
expected change $\E(\phi_t-\phi_{t+1}\mid\phi_t)$ an amount of no less than $-c$.

Combining this with \eqref{eq:upper-bound-above-drift-sqrt}, 
we bound the drift altogether by  
\[ \E(\phi_t-\phi_{t+1}\mid \phi_t) \ge \frac{\gamma\sqrt{\phi_t}}{2}-c.
\]
If $\phi_t$ is above a sufficiently large constant, more precisely, if $ \sqrt{\phi_t}\ge 4c/\gamma$ (equivalent to 
$c\le \gamma\sqrt{\phi_t}/4$), 
the bound is still positive and only by a constant factor smaller than the drift bound stemming 
from~\eqref{eq:upper-bound-above-drift-sqrt}. We obtain 
\[ \E(\phi_t-\phi_{t+1}\mid \phi_t; \sqrt{\phi_t}\ge 4c/\gamma) \ge \frac{\gamma\sqrt{\phi_t}}{2}-c \ge 
\frac{\gamma\sqrt{\phi_t}}{2} - \frac{\gamma\sqrt{\phi_t}}{4} = \frac{\gamma\sqrt{\phi_t}}{4}.
\]
We now set $h(\phi_t)\coloneqq \frac{\gamma\sqrt{\phi_t}}{4}$ and  
apply 
 the variable drift theorem (Theorem~\ref{theo:variable-drift}) with drift function $h(\phi_t)$, maximum $n$ 
and minimum $\xmin=16c^2/\gamma^2$ (since $\sqrt{\phi_t}\ge 4c/\gamma$ is required).
 Hence, the expected number of generations until 
the $\phi$-value is at most $16c^2/\gamma^2$ is at most
\[
\frac{\xmin}{h(\xmin)} + \int_{\xmin}^n \frac{\mathrm{d}x}{h(x)} \le \frac{16c^2/\gamma^2}{\gamma\sqrt{16c^2/\gamma^2}/4} + 
\int_{16c^2/\gamma^2}^n \frac{\mathrm{d}x}{\gamma\sqrt{x}/4} = O(\sqrt{n})
\]
since both $c$ and $\gamma$ are constant. Hence, by Markov's inequality, $O(\sqrt{n})$ generations, 
amounting to $O(\lambda\sqrt{n})$ function evaluations, suffice with probability $\Omega(1)$ to reach 
$\phi_t\le 16c^2/\gamma^2=O(1)$. 
It is easy to see that $\phi_t=O(1)$ implies 
an at least constant probability of sampling the all-ones string (assuming that all $p_{t,i}$ are at least~$1/4$). 
Hence, the optimum 
is sampled in $O(\sqrt{n})$ generations  with probability $\Omega(1)$, which, as outlined above, 
 proves the first statement 
of the lemma and also the statement on \umda's  expected running time.
\end{proofof}

The $\Omega(1)$ bound in Theorem~\ref{theo:upper-above-phase} on 
the probability of \umdastar finding the optimum (without stagnating at wrong borders) 
results from two factors: the application of Markov's inequality and the 
probability $\Omega(1)$ of sampling the optimum after $\sqrt{\phi_t}\le 4c/\gamma$ has been reached from the first time. 
It is not straightforward to improve this bound to higher values (\eg, probability $1-o(1)$) since
one would have to analyze the subsequent development of potential in the case that the algorithm fails 
to sample the optimum at $\sqrt{\phi_t}\le 4c/\gamma$.

\section{Below the Phase Transition}
\label{sec:upperBoundTwo}
Theorem~\ref{theo:upper-above-phase} crucially assumes that $\mu\ge c\sqrt{n}\log n$ for 
a large constant~$c>0$. As described above, the \umda shows a phase transition between unstable 
and stable behavior at the threshold $\Theta(\sqrt{n}\log n)$. Above the threshold, 
the frequencies typically stay well focused on their drift towards the upper border and do not 
 drop much below $1/2$. The opposite is the case if $\mu<c'\sqrt{n}\log n$ for a sufficiently 
small constant~$c'>0$. Krejca and Witt \cite{KrejcaWittFOGA2017} have shown for this regime 
that with high probability 
$n^{\Omega(1)}$ frequencies 
will walk to the lower border before the optimum is found, resulting 
in a coupon collector effect and 
therefore the lower bound $\Omega(n\log n)$ on the running time. It also follows 
directly from their results (although this was not made explicit) that \umdastar will  
in this regime with high probability  have infinite optimization time since 
$n^{\Omega(1)}$ frequencies will get stuck at~$0$. Hence, in the regime $\mu=\Theta(\sqrt{n}\log n)$, 
the \umdastar turns from efficient with at least constant probability to inefficient 
with overwhelming probability.

Interestingly, the value $\Theta(\sqrt{n}\log n)$ has also been derived in \cite{SudholtWitt2016} 
as an important parameter setting 
\wrt the update strengths called $K$ and $1/\rho$ 
in the simple EDAs cGA and $2$-MMAS$_\mathrm{ib}$, respectively. Below the threshold value, lower bounds are obtained
through a coupon collector argument, whereas above the threshold, the running time is 
$O(K\sqrt{n})$ (and $O((1/\rho)\sqrt{n})$, respectively) 
since frequencies evolve smoothly towards the upper border. The \umda and \umdastar demonstrate the same 
threshold behavior, even at the same threshold points.

The EDAs considered in \cite{SudholtWitt2016} use borders~$1/n$ and $1-1/n$ 
for the frequencies in the same way as the \umda. The only upper bounds on the running time are obtained for update strengths  
greater than $c\sqrt{n}\log n$. Below the threshold, no conjectures on upper bounds on the running time are stated; however, 
 it seems that the authors do not see any benefit in smaller settings of the parameter since they recommend always to 
choose values above the threshold. Surprisingly, this does not seem to be necessary if the borders $[1/n,1-1/n]$ 
are used. With respect to the \umda, we will show that  even for logarithmic~$\mu$ it has polynomial expected running time, 
thanks to the borders, while we already know that \umdastar will fail. 
We also think that a similar effect can  be shown for the EDAs in  \cite{SudholtWitt2016}.

We now give our theorem for the \umda with small~$\mu$. If $\mu=\Omega(\sqrt{n}\log n)$, it is weaker 
than Theorem~\ref{theo:upper-above-phase}, again underlining the phase transition. The proof
is more involved since it  has to carefully  bound the number of times frequencies 
leave a border state.

\begin{theorem}
\label{theo:upper-below-phase}
Let %
$\lambda=(1+\beta)\mu$ for an arbitrary constant~$\beta>0$ and $\mu\ge c\log n$ 
for a sufficiently large constant~$c>0$
 as well as $\mu = O(n^{1-\epsilon})$ for some constant~$\epsilon>0$.
 Then the 
expected optimization time of \umda on \OneMax is $O(\lambda n)$.  
For \umdastar, it is infinite with high probability 
if $\mu<c'\sqrt{n}\log n$ for a sufficiently small constant~$c'>0$.
\end{theorem}

Before we prove the theorem, we state two lemmas that 
work out properties of frequencies that have reached the upper border. In a nutshell, 
the following lemma show that with high probability such frequencies will  
stay in the vicinity of the border afterwards, assuming 
a sufficiently large drift stemming from $\sigma_t=O(1)$. Moreover, it return to the 
border very quickly afterwards with high probability; in fact, it will only spend a constant number of  
steps at a non-border value with high probability.   
Since we are dealing with a Markov chain, 
 the analysis does not change if a certain amount of time has elapsed. In the following lemma, 
we therefore \wlo\ assume that the  time where the border is hit equals~$0$.
%

\begin{lemma}
\label{lem:upper-below-phase-above-half}
In the setting of Theorem~\ref{theo:upper-below-phase}, 
consider the frequency $p_t$, $t\ge 0$, belonging to an arbitrary but 
fixed bit and suppose that $p_0=1-1/n$ and $\sigma_t=O(1)$ for 
$t\ge 1$, where $\sigma_t^2=\sum_{i=1}^n p_{t,i}(1-p_{t,i})$.
Then for every constant $c_1>0$ there is a constant~$c_2>0$ such that 
\begin{itemize}
\item $p_1\ge 1-c_2/\mu$ with probability at least $1-n^{-c_1}$.  
\item There is a constant~$r>0$ with the following properties: if $p_1\ge 1-c_2/\mu$ then  
with probability $1-e^{-\Omega(c\log n)}$, $p_t\ge 1-c_2/\mu$ 
for all $t\in\{2,\dots,r-1\}$ and finally $p_r=1-1/n$. 
\end{itemize}

%
\end{lemma}

\begin{proof}
By assumption, $p_0=1-1/n$. We analyze the 
distribution of $p_{1}$.  
Since the $\mu$ best individuals are biased towards 
one-entries, the number~$N$ of $0$s sampled at the bit among the $\mu$ best 
is stochastically smaller than $\Bin(\mu,1/n)$. Since  $\mu=O(n^{1-\epsilon})$ according to the assumptions from
Theorem~\ref{theo:upper-below-phase}, we obtain for any $k>0$ 
\[
\Prob(N\ge k)\le \binom{\mu}{k} \left(\frac{1}{n}\right)^k \le \frac{1}{k!}\left(\frac{\mu}{n}\right)^k \le n^{-c' k}
\]
for some constant~$c'>0$. If $N=k>0$, then clearly $p_1=1-k/\mu$ by the update rule of the \umda (using that 
$\mu=o(n)$ so that $p_1$ is not capped at the border). 
We can now choose a constant~$k$ such 
that $c'k=c_1$  
and establish the first claim with $c_2=k$.   

To establish the second claim, 
we consider the distance $X_t\coloneqq \mu p_{t}$ of the frequency 
from the lower border for $t\ge 1$. By assumption, $X_1=\mu-c_2$. 
The aim is to show the following: a phase of $O(1)$ steps will consist 
of increasing steps only such that the upper border is finally reached, all 
with probability $1-e^{-\Omega(c\log n)}$.  For technical reasons, 
it is not straightforward to apply Theorem~\ref{theo:negative-drift-scaling-2017} 
to show this.
Instead, 
we use a more direct and somewhat simpler argumentation based on the analysis of the sampling process.

\newcommand{\tmu}{\tilde{\mu}}
Recall from Section~\ref{sec:first-and-second-class}, in particular Lemma~\ref{lem:2nd-lower-part-2},  
that $X_{t+1}$ is obtained by summing up 
the number of $1$s at position~$j$ 
sampled in both the $1$st-class individuals and the $2$nd-class 
individuals at generation~${t+1}$. Denote by $D^*$ the number of $2$nd-class candidates 
and by $C^{**}$ the number of 2nd-class individuals as in 
Lemma~\ref{lem:2nd-lower-part-1}. From Lemma~\ref{lem:2nd-lower-part-1b} 
we obtain that $D^*-C^{**}-1$ with probability $1-e^{-\Omega(\mu)}$ follows a binomial 
distribution with parameters $\Omega(\mu)$ and $\Omega(1/\sigma_t)$; analogously for $C^{**}$. We assume both 
to happen.  Note that $\Omega(1/\sigma_t)=\Omega(1)$ by assumption, which implies 
$\mu/\sigma_t = \Omega(\mu) = \Omega(c \log n)$, where we used 
the constant~$c$ from Theorem~\ref{theo:upper-below-phase}.

By Chernoff bounds the probability that $D^*-C^{**} \ge c_3 \mu $ for some constant~$c_3>0$  
is at least $1-e^{-\Omega(\mu)} = 1-e^{-\Omega(c\log n)}$; similarly, $C^{**} \ge c_3 \mu$ with probability 
$1-e^{-\Omega(c\log n)}$. 
 We assume both these event also to happen. The number of $1$s 
sampled in the $D^*$ candidates is at least $D^* (1-c_4) X_t/\mu$ with 
probability $1-e^{-\Omega(c\log n)}$ by Chernoff bounds (using that $X_t/\mu=\Theta(1)$), where 
$c_4$ is a constant that can be chosen small enough. Also by Chernoff bounds, the number of $1$s 
sampled in the $\mu-C^{**}$ $1$st-class individuals (pessimistically assuming 
that $C^*=C^{**}$)  is at least 
$(\mu-C^{**})  X_t/\mu - c_4 \mu$ with probability $1-e^{-\Omega(c\log n)}$.

Assuming all 
this to happen, we have $X_{t+1}\ge X_t -c_4 \mu  + (D^*-C^{**}) (1-c_4) X_t/\mu - c_4 C^{**} X_t/\mu $ 
(or $X_{t+1}$ takes its maximum $\mu$ anyway) 
with probability $1-e^{-\Omega(c\log n)}$. If $c_4$ is chosen sufficiently 
small compared to~$c_3$, we obtain $X_{t+1}\ge X_t+ c_4\mu /2$ with probability 
$1-e^{-\Omega(c\log n)}$. Note that the constants in the exponent of $e^{-\Omega(c\log n)}$ 
may be different; we use the smallest ones from these estimations and union bounds 
to obtain the bound $1-e^{-\Omega(c\log n)}$ on the  final probability.

If $X_{t+1}\ge X_t+ c_4\mu /2$ for a constant number $r\le 2c_2/c_4$ of iterations, then 
the frequency~$p_t$ has raised from~$1-c_2/\mu$ to its upper border  in this number 
of steps. By a union bound over these many iterations, the probability 
that any of the events necessary for this happens to fail is most $O(1)\cdot e^{-\Omega(c\log n)}=e^{-\Omega(c\log n)}$ 
as claimed above.
\qed\end{proof}

In the following lemma, we show a statement on the expected value of the frequency.  
To obtain this value, it is required that the frequency does not drop below~$1-O(1/\mu)$ and 
quickly returns to $1-1/n$, 
which can be satisfied by means of the previous lemma.

\begin{lemma}
\label{lem:markov-chain-analysis}
In the setting of Theorem~\ref{theo:upper-below-phase}, 
consider the frequency $p_t$, $t\ge 0$, belonging to an arbitrary but 
fixed bit and suppose that \begin{itemize}
\item $p_0=1-1/n$, 
\item $p_t\ge 1-O(1/\mu)$ for 
all $t\ge 0$, and
\item there is a constant~$r>0$ such that the following holds: 
for all $t'$ where $p_{t'}=1-1/n$ there is some $s\in\{1,\dots, r\}$ 
such that $p_{t'+s}=1-1/n$. 
\end{itemize}
 Then for all $t\ge 0$ it holds that  $\E(p_t) = 1-O(1/n)$.
\end{lemma}

\begin{proof}
By assumption, $p_0=1-1/n$. We will analyze the 
distribution and expected value 
of $p_{t}$ for $t\ge 1$. 
We will show that for all $t\ge 0$, we have $\Prob(p_t=1-1/n) \ge 1-c_1\mu/n$ for 
some constant~$c_1>0$. By assumption, in any case $p_t\ge 1-O(1/\mu)$. Using the 
law of total probability to combine the cases $p_t=1-1/n$ and $p_t<1-1/n$, 
we obtain 
\[
\E(p_t)\ge \left(1-\frac{c_1\mu}{n}\right)\left(1-\frac{1}{n}\right) + \frac{c_1\mu}{n} \frac{\mu-O(1)}{\mu} = 1-O(1/n).
\]

We are left with the claim $\Prob(p_t=1-1/n)=1-O(\mu/n)$. Recall that 
on $p_{1}<1-1/n$ we have $p_{1}=1-O(1/\mu)$ and that again $p_t=1-1/n$ after $t\le r$ steps. 
We take now a simpler view by means of a two-state Markov chain, where state~$0$ 
corresponds to frequency $1-1/n$ and state~$1$ to the rest of the reachable states. Time is 
considered in blocks of~$r$ steps, which will be justified in the final paragraph of this proof. 
The transition probability from $0$ to~$1$ is at most $O(\mu/n)$ (the expected number of $0$s sampled 
in the $\mu$ best individuals)  
and the transition probability from $1$ to~$0$ is $1$; the remaining probabilities are self-loops. Now, it 
is easy to analyze the steady-state probabilities, which are $1-O(\mu/n)$ for state~$0$ and 
$O(\mu/n)$ for state~$1$. Moreover, since the chain starts in state~$0$, simple 
calculations of occupation probabilities over time yield for state~$0$ 
a probability of  
$1-O(\mu/n)$ for all points of time $t\ge 0$. More precisely, at the transition 
from time $t$ to time~$t+1$ the occupation probability of state~$0$ can only decrease by $O(\mu/n)$. 
When  state~$1$ exceeds an occupation probability of $c_2\mu /n$ for a sufficiently large constant 
$c_2>0$, the process goes to state~$0$ with probability at least $c_2\mu/n$, which is 
less than the decrease of the occupation probability for state~$0$ for $c_2$ large enough. Hence, 
the occupation probability of state~$0$ cannot drop below~$1-O(\mu/n)$. 

Finally, we argue why we may consider phases of length~$r$ in the Markov chain analysis. Note that 
only every $r$th step a transition from state~$1$ to $0$ is possible (in our pessimistic model), however, 
in fact every step can transit from state~$0$ to~$1$. Formally, we have 
to work in these additional steps in our two-state model. 
We do so by increasing the probability 
of leaving state~$0$ by a factor of~$r$, which vanishes in the $O(\mu/n)$ bound used above.
\qed\end{proof}

Having proved these two preparatory lemmas, we can give the proof of the main theorem from this section.

\begin{proofof}{Theorem~\ref{theo:upper-below-phase}}
The second statement can be derived from \cite{KrejcaWittFOGA2017}, as  discussed above. We now focus 
on the first claim, basically reusing  the potential function $\phi_t =n-1-\sum_{i=1}^n p_{t,i}$ from the proof 
of Theorem~\ref{theo:upper-above-phase}. Let $k$ denote the 
number of frequencies below $1-c/n$ for the $c$ from Lemma~\ref{lem:2nd-lower-part-3}, w.\,l.\,o.\,g., these are the 
frequencies associated with bits $1,\dots,k$. 
The last $n-k$ bits are actually at $1-1/n$ since $1/\mu=\omega(1/n)$ by assumption. They are 
set to~$0$ with probability at most $1/n$ in each of the selected offspring, amounting 
to a total expected loss of at most~$1$. Similarly as in the proof of Theorem~\ref{theo:upper-above-phase}, we 
compute the drift
\begin{align}
\notag
\E(\phi_t-  \phi_{t+1}\mid \phi_t)  & \ge \sum_{i=1}^k (p_{t+1,i}-p_{t,i}) - (n-k)\frac{1}{n}\\\notag
& \ge \sum_{i=1}^k (p_{t,i}+\gamma 
p_{t,i}(1-p_{t,i})/\sigma_t - p_{t,i}) - 1 \\\notag
& =  \frac{\gamma\sum_{i=1}^k p_{t,i}(1-p_{t,i})}{\sqrt{((n-k)/n)(1-1/n)+\sum_{i=1}^k p_{t,i}(1-p_{t,i})}} -1 \\
&  \ge \frac{\gamma\sum_{i=1}^k p_{t,i}(1-p_{t,i})}{\sqrt{1+\sum_{i=1}^k p_{t,i}(1-p_{t,i})}}    - 1
\label{eq:drift-below-phase}
\end{align}
where $\gamma$ is the implicit constant in the 
$\Omega$-notation from Lemma~\ref{lem:2nd-lower-part-3}, and  
the last equality just used the definition of $\sigma_t$. We now distinguish two cases 
depending on $V^*\coloneqq  \sum_{i=1}^k p_{t,i}(1-p_{t,i})$, the total variance \wrt\ the bits not at the upper border. 

\textbf{Case 1:}
If $V^*\ge c'$ for some sufficiently large constant~$c'>0$, we obtain 
\[
\frac{\gamma V^*}{\sqrt{1+V^*}} \ge 2,
\]
and therefore 
\[\E(\phi_t-\phi_{t+1}\mid \phi_t)\ge \frac{\gamma V^*}{\sqrt{1+V^*}}-1 \ge  
1\]
from \eqref{eq:drift-below-phase}. If $V^*<c'$, we will show by advanced arguments that the bits 
that have reached the upper border can almost be ignored and that 
the drift with respect to the other bits is still in the order $\Omega(V^*/\sqrt{1+V^*})$.  Using this (to be proved) statement, 
we apply variable drift (Theorem~\ref{theo:variable-drift}) with $\xmin=1/\mu$ (since each $p_{i,t}=i/\mu$ for some $i\in\{1,\dots,\mu-1\}$ if 
it is not at a border) and \[h(x)\coloneqq \min\{1, c''x/\sqrt{1+x}\}\] for some constant~$c''$. Let $x^*$ be the point where 
$1=c''x^*/\sqrt{1+x^*}$ and note that $x^*$ is some constant bigger than~$1$ if $c''$ is small enough. We obtain the upper bound
\begin{align}
 \frac{\xmin}{h(\xmin)} + \int_{1/\mu}^n \frac{1}{h(x)}\mathrm{d}x = \frac{(1/\mu)\sqrt{1+1/\mu}}{c''/\mu} + \int_{1/\mu}^{x^*} \frac{\sqrt{1+x}}{c''x}\mathrm{d}x 
+ \int_{x^*}^n \frac{\mathrm{d}x }{1} 
\label{eq:upper-bound-above-integral}
\end{align}
on the expected number of generations until all frequencies have hit the upper border at least once. 
The anti-derivative of $\sqrt{1+x}/x$ is $2\sqrt{1+x}+\ln(\sqrt{1+x}-1) - \ln(\sqrt{1+x}+1)$. 
Hence, \[
\int_{1/\mu}^{x^*} \frac{\sqrt{1+x}}{c''x}\mathrm{d}x \le O(1) - \frac{\ln(\sqrt{1+1/\mu}-1)}{c''} = O(\log \mu)
\]
since $\ln(\sqrt{1+1/\mu}-1)\ge \ln(1/(2\mu)-1/(8\mu^2)) = -\ln(\mu)/2+\Theta(1)$ by 
a Taylor expansion. Hence, the whole bound \eqref{eq:upper-bound-above-integral} can be simplified 
to \[
O(1) + O(\log\mu) + O(n) = O(n)
\]
using $\mu=o(n)$. When the potential has reached its minimum value, the optimum is sampled 
with probability~$\Omega(1)$. If this fails, the argumentation can be repeated. The expected number 
of repetitions is $O(1)$.  
This corresponds to an expected running time of $O(\lambda n)$.

\textbf{Case 2: }
We still have to show that we have a drift of $\Omega(V^*/\sqrt{1+V^*})$ if $V^*\le c'$. Actually, 
we will consider a phase of $\kappa n$ generations for some sufficiently large constant~$\kappa>0$ 
and show that the claim holds with high probability throughout the phase.
We then show that under this assumption the optimum is still sampled with probability~$\Omega(1)$ in the phase. 
In case of a failure, we repeat the argumentation and get an expected number of $O(1)$ repetitions, altogether 
an expected running time of $O(\lambda n)$.

We have 
seen above that frequencies at the upper border may contribute negatively to the drift of the $\phi_t$-value. 
Hence, to show the claim that the potential also is expected to decrease when $V^*\le c'$, 
we will analyze an arbitrary but fixed 
frequency and use the above-proven fact that 
it is likely to stay in the vicinity of the upper border once 
having been there. We work under the assumption that we have  
$\sigma_t=O(1)$ within a phase of $\kappa n$ generations. We claim that this actually happens with 
probability at least~$1-O(1/n)$, a proof of which will be given below. In case of a failure, we repeat 
the argumentation and are done within an expected number of $O(1)$ repetitions.

We invoke Lemma~\ref{lem:upper-below-phase-above-half} for any frequency that has hit the upper border   
and note that, unless a failure event of probability $e^{-\Omega(c\log n)}+n^{-c_1}$ 
happens, we can apply Lemma~\ref{lem:markov-chain-analysis}. We assume that the constant~$c$ 
(stemming from  the assumption~$\mu\ge c\log n$) as well as~$c_1$ 
have been chosen appropriately such that the failure probability is $O(1/n^2)$. Hereinafter, we 
assume that no failure occurs. 
Hence, we have an expected frequency of $1-O(1/n)$. This expected value is 
just the probability of sampling a~$1$ at the underlying bit. 
 Consequently, if there are $\ell$ bits 
that have been at the upper border at least once, the probability of sampling only $1$s at all these 
bits is at least \[\prod_{i=1}^\ell \left(1-O\left(\frac{1}{n}\right)\right)=\Omega(1).\] 
This still allows the 
optimum to be sampled with probability~$\Omega(1)$ 
after the potential on the bits that never have hit the border so far has decreased below~$c'$.

Finally, we have to justify why with high probability 
$\sigma_{t'}=O(1)$ for any $t'\ge t$ within $\kappa n$ steps 
after the first time~$t$ where $\phi_t=O(1)$. The frequencies 
that never have been at the upper border 
contribute at most $c'=O(1)$ to~$\phi_t$ 
by assumption and, since $\sigma_t\le \phi_t$, also to~$\sigma_t$. 
Frequencies that are  at the upper border leave this state only with probability $O(\mu/n)$ 
and have a value of $1-O(1/\mu)$ afterwards according to Lemma~\ref{lem:upper-below-phase-above-half} with 
high probability. In every step only an expected number of 
$n\cdot O(\mu/n)=O(\mu)$ frequencies leave the upper border. By Chernoff bounds, the number is $O(\mu)$ 
even with probability $1-e^{-\Omega(\mu)}$. 
Finally, 
since every frequency that has left the upper border returns to it within~$r$ steps with high 
probability, there are within the phase only 
$O(\mu)$ such frequencies with high probability. 
%
Their contribution to $\sigma_t$ is therefore $O(\mu)\cdot O(1/\mu)=O(1)$ with high probability. The 
failure probability is again $O(1/n^2)$ as the constants in Lemma~\ref{lem:upper-below-phase-above-half}
and the $c$ in the assumption~$\mu\ge c\log n$ were chosen sufficiently large. 
Hence, the probability of a failure in $O(\kappa n)$ steps 
is $O(1/n)$.   

We finally note that, although the state of the algorithm may switch between the cases $V^*\le c'$ and $V^*>c'$
more than once, the drift argument can always be applied since we have established a drift of 
$ \Omega(V^*/\sqrt{1+V^*})$ for the $\phi_t$-value regardless of the case.
\end{proofof}

We have now concluded the proof of Theorem~\ref{theo:upper-below-phase}. 
As mentioned before, we can extract from this theorem 
 a second value of $\mu$ that gives the 
$O(n\log n)$ running time bound, namely $\mu=c'\log n$. We also believe that 
values $\mu=o(\log n)$ will lead to a too coarse-grained frequency scale 
and exponential lower bounds on the running time, which can be regarded 
as another phase transition in the behavior. We do not give a proof here 
but only mention that such a phase transition from polynomial to exponential 
running time is known from ACO algorithms and non-elitist $(1,\lambda)$~EAs when 
a parameter crosses $\log n$ \cite{Neumann2010a,RoweSudholtTCS2014}.

\section{Experiments}
We have carried out experiments for \umda on \om to gain some empirical insights into the relationship 
between $\lambda$ and the average running time. The algorithm was implemented in the C programming language 
using the PCG32 random number generator. 
The problem size was set to $n=2000$, $\lambda$ was increased from $14$ to $350$ in steps of size~$2$, $\mu$ 
was set to $\lambda/2$, and, due to the high variance of the runs especially for small $\lambda$, 
an average was taken over $3000$ runs for every setting of~$\lambda$. The left-hand side 
of Figure~\ref{fig:eda:experiments-1} demonstrates that the average running time in fact shows a multimodal dependency on~$\lambda$. 
Starting out from very high values, it takes a minimum at $\lambda\approx 20$ and then increases again up to $\lambda\approx 70$. Thereafter 
 it falls 
again up to $\lambda\approx 280$ and finally increases rather steeply for the rest of the range. The right-hand side, a
semi-log plot, also illustrates that 
the number of times the lower border is hit seems to decrease exponentially with $\lambda$. The phase transition where the behavior 
of frequencies turns from chaotic into stable is empirically located somewhere between $250$ and $300$.

Similar results are obtained for $n=5000$, see Figure~\ref{fig:eda:experiments-2}. The location of the  maximum 
does not seem to increase linearly with~$n$.

\pgfplotstableread[col sep = semicolon]{umda2000-3000.txt}\mydata
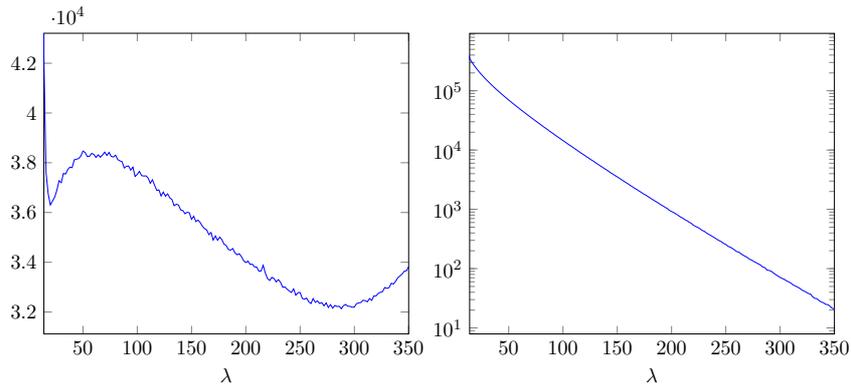
\begin{figure}
\centerline{
\begin{tikzpicture}[scale=0.7]
  \begin{axis}[
    legend pos = north east, 
    xmin = 14, xmax = 350, xlabel = $\lambda$
    ]
    \addplot table[x index = {0}, y index = {1}, mark = none]{\mydata};
		\end{axis}
		\begin{scope}[xshift=8 cm]
		\begin{semilogyaxis}[
    legend pos = north east, 
    xmin = 14, xmax = 350, xlabel = $\lambda$
    ]
		    \addplot table[x index = {0}, y index = {2}, mark = none]{\mydata};
  \end{semilogyaxis}
	\end{scope}
\end{tikzpicture}}

\caption{Left-hand side: empirical running time of \umda on \om, right-hand side: number of hits of lower border; for $n=2000$, $\lambda\in\{14,16,\dots,350\}$, $\mu=\lambda/2$,  and averaged over 
$3000$ runs}
\label{fig:eda:experiments-1}
\end{figure}

\pgfplotstableread[col sep = semicolon]{umda5000-3000.txt}\mydata
\begin{figure}
\centerline{
\begin{tikzpicture}[scale=0.7]
  \begin{axis}[
    legend pos = north east, 
    xmin = 16, xmax = 650, xlabel = $\lambda$
    ]
    \addplot table[x index = {0}, y index = {1}, mark = none]{\mydata};
		\end{axis}
		\begin{scope}[xshift=8 cm]
		\begin{semilogyaxis}[
    legend pos = north east,
    xmin = 16, xmax = 650, xlabel = $\lambda$
    ]
		    \addplot table[x index = {0}, y index = {3}, mark = none]{\mydata};
  \end{semilogyaxis}
	\end{scope}
\end{tikzpicture}}

\caption{Left-hand side: empirical running time of \umda on \om, right-hand side: number of hits of lower border; for $n=5000$, $\lambda\in\{14,16,\dots,650\}$, $\mu=\lambda/2$,  and averaged over 
$3000$ runs}
\label{fig:eda:experiments-2}
\end{figure}
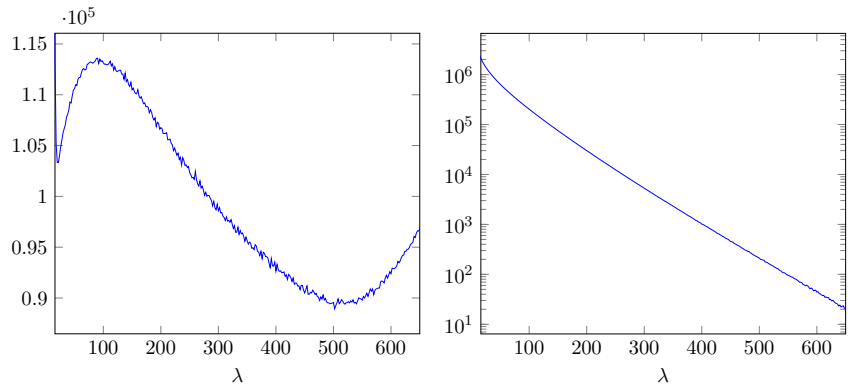

\section*{Conclusions}
We have analyzed the UMDA on $\om$ and obtained the upper bounds  
$O(\mu\sqrt{n})$ and $O(\mu n)$ on its expected running time in different 
domains for $\mu$, more precisely if $\mu\ge c\sqrt{n}\log n$  and $\mu\ge c'\log n$, respectively, 
where $c,c'$ are positive constants.  This 
implies an expected running time of $O(n\log n)$ for two asymptotic values of $\mu$, closing the previous 
gap between the lower bound $\Omega(\mu\sqrt{n}+n\log n)$ and the upper bound $O(n\log n\log\log n)$. 
In our proofs, we provide detailed tools for the analysis of the stochastic processes at single 
frequencies in the \umda. We hope that these tools will be fruitful in future analyses of EDAs.

 We note that all our results assume $\lambda=\bigO(\mu)$. However, we do not think that 
 larger~$\lambda$ can be beneficial; if $\lambda=\alpha \mu$, for $\alpha=\omega(1)$, the 
progress due to $2$nd-class individuals can be by a factor of at most $\alpha$ bigger; however, also 
the computational effort per generation would grow by this factor. A formal analysis of 
other ratios between $\mu$ and $\lambda$ is open, as is the case of sublogarithmic~$\mu$. 
Moreover, we do not have lower bounds matching the upper bounds  
from Theorem~\ref{theo:upper-above-phase} if $\mu$ is in the regime where both $\mu=\omega(\log n)$ and 
$\mu=o(\sqrt{n}\log n)$.

\paragraph{Acknowledgments} 
Financial support by the Danish Council for Independent Research 
(DFF-FNU 4002--00542) is gratefully acknowledged.

\appendix
\section{Appendix}

\subsection{Proof of Theorem~\ref{theo:negative-drift-scaling-2017}}

We will use Hajek's drift 
theorem to prove  Lemma~\ref{theo:negative-drift-scaling-2017}.  
As we are dealing with a stochastic process, we 
implicitly assume that the random variables~$X_t$, $t\ge 0$,  
are adapted to some filtration $\filtt$ such as the natural filtration~$X_0,\dots,X_t$, $t\ge 0$.

We do not formulate the theorem using a potential/Lyapunov function $g$ 
mapping from some state space to the reals either. Instead,  
we \wlo\ assume the random variables~$X_t$ as already obtained by the 
mapping.

The following theorem follows immediately from taking Conditions~D1 and D2 in \cite{Hajek1982} 
and applying Inequality~(2.8) in a union bound over $L(\ell)$ time steps.

\begin{theorem}[\cite{Hajek1982}]
  \label{theo:orig-drift}
  Let $X_t$, $t\ge 0$, be real-valued random variables describing a
	stochastic process over some state space, adapted to a filtration $\filtt$. Pick two
  real numbers $a(\ell)$ and $b(\ell)$ depending on a parameter~$\ell$
  such that $a(\ell)<b(\ell)$ holds. Let $T(\ell)$ be the random
  variable denoting the earliest point in time $t\ge 0$ such that
  $X_t\le a(\ell)$ holds.  If there are $\lambda(\ell)>0$ and
  $p(\ell)>0$ such that the condition
  \begin{equation}
  \label{eq:maindriftcondition}
  \tag{$\ast$}
  \E\bigl(e^{-\lambda(\ell)\cdot (X_{t+1}-X_t)}
     \mid \filtt \,;\, a(\ell)<  X_t < b(\ell)\bigr)
  \;\,\le\,\; 1-\frac{1}{p(\ell)}  
  \end{equation}
  holds for all $t\ge 0$ then for all time bounds $L(\ell)\ge 0$
  \[
  \Prob\bigl(T(\ell)\le L(\ell) \mid X_0\ge b(\ell)\bigr) 
  \;\le\; e^{-\lambda(\ell)\cdot
    (b(\ell)-a(\ell))}\cdot L(\ell)\cdot D(\ell)\cdot p(\ell),
  \]
  where $D(\ell)=\max\bigl\{1,\E\bigl(e^{-\lambda(\ell)\cdot (X_{t+1}-b(\ell))}\mid \filtt \,;\, X_t\ge b(\ell)\bigr)\bigr\}$.
\end{theorem}

\begin{proof}[Proof of {Theorem~\ref{theo:negative-drift-scaling-2017}}]
  We will apply Theorem~\ref{theo:orig-drift} for suitable choices of
  its variables, some of which might depend on the parameter
  $\ell=b-a$ denoting the length of the interval~$[a,b]$.
  The following argumentation is
  also inspired by Hajek's work \citep{Hajek1982}.

	By assumption $\Delta_t(X_{t+1}-X_t)\preceq X_{t+1}-X_t$. Clearly, 
	for the process $X'_t=X_0+\sum_{j=0}^{t-1} \Delta_t(X_{t+1}-X_t)$ 
	we have $X_t'\preceq X_t$. Hence, the hitting time $T^*$ for 
	state less than~$a$ of the original 
	process $X_t$ is stochastically 
	at least as big as the corresponding hitting time of the process~$X_t'$. In the following, 
	we will therefore without further mention analyze $X_t'$ instead of~$X_t$ and bound the tail of its hitting time. 
	We work with $\Delta\coloneqq \Delta_t$, which equals $X'_{t+1}-X'_t$. We still use the old notation 
	$X_t$ instead of $X'_t$.

	%
	The aim is to bound 
	the moment-generating function (mgf.) from 
	 Condition~\eqref{eq:maindriftcondition}. In this analysis,  we for notational convenience often omit 
	the filtration $\filtt$. 
	First we observe that 
	it is sufficient to bound the mgf. of $\Delta\cdot \indic{\Delta\le \kappa\epsilon}$ 
	since 
	\begin{align*}
	\E(e^{-\lambda \Delta}) & = 
	\E(e^{-\lambda \Delta \indic{\Delta\le \kappa\epsilon} - \lambda\Delta \indic{\Delta> \kappa\epsilon}})\\
  & = \E(e^{-\lambda \Delta \indic{\Delta\le \kappa\epsilon}} e^{ - \lambda\Delta \indic{\Delta> \kappa\epsilon}})
  \le 	\E(e^{-\lambda \Delta \indic{\Delta\le \kappa\epsilon}}), 
	\end{align*}
	using $\Delta \indic{\Delta> \kappa\epsilon}>0$ and 
	hence $e^{ - \lambda\Delta \indic{\Delta> \kappa\epsilon}}\le 1$. 
	In the following, we omit the factor $\indic{\Delta\le  \kappa\epsilon}$ 
	but implicitly multiply $\Delta$ with it all the time. The same goes 
	for $\indic{a<X_t<b}$. 
	
	To establish Condition~\eqref{eq:maindriftcondition}, it is sufficient 
  to identify values $\lambda\coloneqq \lambda(\ell)>0$ and $p(\ell)>0$ such
  that
  \[
  \E(e^{-\lambda \Delta}\indic{a<X_t<b})
  \;\le\; 1-\frac{1}{p(\ell)}.
  \]
  Using the series expansion of the exponential function, we get
  \begin{align*}
    & \E(e^{-\lambda \Delta}\indic{a<X_t<b})  = 
		1 - \lambda \E(\Delta) + \sum_{k=2}^\infty  \frac{(-\lambda)^{k}}{k!} \E(\Delta^k)\\ 
    & \quad = 
    1 - \lambda \E(\Delta) + \sum_{k=2}^\infty  \frac{(-\lambda)^{k}}{k!} \left(\E(\Delta^k\indic{\Delta\ge 0}) + 
		\E(\Delta^k\indic{\Delta< 0})\right). 
		\end{align*}
		
		We first concentrate on the positive steps in the direction of the expected value, more precisely, we consider
		for any odd $k\ge 3$ 
		\[
		M_k \coloneqq \frac{\lambda^{k}}{k!} \E(\Delta^k\indic{\Delta\ge 0}) - \frac{\lambda^{k+1}}{(k+1)!} \E(\Delta^{k+1}\indic{\Delta\ge 0}).
	\]
	Since we implicitly multiply with $\indic{\Delta\le \kappa\epsilon}$, we have $\Delta^k\indic{\Delta\ge 0}\le (\kappa\epsilon)^k$ and hence 
	$\lvert\E(\Delta^{k+1}\indic{\Delta\ge 0})/\E(\Delta^k\indic{\Delta\ge 0})\rvert \le \kappa\epsilon$. By choosing $\lambda\le 1/(\kappa\epsilon)$, 
	we have 
	\[
	M_k\ge \frac{\lambda^k}{k!} \E(\Delta^k\indic{\Delta\ge 0}) - 
	\frac{\lambda^{k}}{\kappa\epsilon (k+1)!} \kappa\epsilon\E(\Delta^{k}\indic{\Delta\ge 0}) \ge 0,
	\]
	for $k\ge 3$ since $(1/k!)/(1/(k+1)!)=k$. Hence,
	\begin{align*}
     \E(e^{-\lambda \Delta}) & \le 1- \lambda\E(\Delta) + \frac{\lambda^2 }{2} \E(\Delta^2\indic{\Delta\ge 0}) \\
		& \le 1- \lambda\E(\Delta) + \frac{\lambda^2 }{2} \E(\Delta\cdot \kappa\epsilon\cdot \indic{\Delta\ge 0})\\
		& \le 1-\lambda\E(\Delta) + \lambda \frac{1}{2\kappa\epsilon}\cdot \kappa\epsilon \cdot \E(\Delta)  
		\le 1-\lambda\epsilon/2
	\end{align*}
	where the first inequality used that $\Delta^{2}\le \Delta\kappa\epsilon$ due to our implicit multiplication 
	with $\indic{\Delta\le \kappa\epsilon}$ everywhere and the second 
	used again $\lambda\le 1/(\kappa\epsilon)$.
	So,  we 
	have estimated the contribution of all the positive steps by $1-\lambda\E(\Delta)/2$. 
	
	We proceed with the remaining terms. We overestimate the sum by using $\Delta'\coloneqq \lvert \Delta \cdot \indic{\Delta<0})\rvert$ 
	and bounding $(-\lambda^k)\le \lambda^k$ in 
	all terms starting from $k=2$. Incorporating the contribution of the positive steps, we   
   obtain for all $\gamma\ge
  \lambda$
  \begin{align*}
  \E(e^{-\lambda \Delta})  & \le
    1 - \frac{\lambda}{2} \E(\Delta) + \frac{\lambda^2}{\gamma^2} \sum_{k=2}^\infty \frac{\gamma^{k}}{k!} \E(\Delta'^k)\\
     & \le
		1 - \frac{\lambda}{2} \E(\Delta)+ \frac{\lambda^2}{\gamma^2} \sum_{k=0}^\infty \frac{\gamma^{k}}{k!} \E(\Delta'^k) 
		 \le 1 -\frac{\lambda}{2} \epsilon + \lambda^2  \underbrace{\frac{\E(e^{\gamma \Delta'})}{\gamma^2}}_{=:C(\gamma)},
  \end{align*}
	where the last inequality uses the first condition of the theorem, \ie, the bound on the drift.
 
  Given any $\gamma>0$, choosing $\lambda\coloneqq \min\{1/(\kappa\epsilon), \gamma,
  \epsilon/(4C(\gamma))\}$ results in
  \[
  \E(e^{-\lambda \Delta}\indic{a<X_t<b})  \;\le\; 1- \frac{\lambda}{2}\epsilon + \lambda\cdot \frac{\epsilon}{4C(\gamma)}\cdot C(\gamma) 
  \;=\; 1-\frac{\lambda \epsilon}{4} \;=\;  1-\frac{1}{p(\ell)}
  \]
  with $p(\ell)\coloneqq 4/(\lambda \epsilon)$.

	The aim is now to choose $\gamma$ in such a way that $\E(e^{\gamma \Delta'})$ is bounded from 
	above by a constant.
	We get
	\[
	\E(e^{\gamma \Delta'}) \;\le\; \sum_{j=0}^\infty e^{\gamma (j+1)r} \Prob(\Delta \le -jr) 
	 \;\le\; 
	\sum_{j=0}^\infty e^{\gamma (j+1)r} e^{- j }  
	\]
	where the inequality uses the second condition of the theorem.

  Choosing $\gamma\coloneqq 1/(2r)$ yields
	 \begin{align*}
	\E(e^{\gamma \Delta'})
    & \le \sum_{j=0}^\infty e^{(j+1)/2 - j }
     = e^{1/2} \sum_{j=0}^\infty e^{-j/2}
		 = e^{1/2} \frac{1}{1-e^{-1/2}}
		\le 4.2.
  \end{align*}
	Hence, $C(\gamma)\le 4.2/\gamma^2$ and
  therefore 	
	$\lambda\le \epsilon/(4\cdot 4.2r^2)< \epsilon/(17r^2)$. 
  From the definition of~$\lambda$, we altogether have 
	 $\lambda = \min\{1/(2r),\epsilon/(17r^2),1/(\kappa\epsilon)\}$. 
	Since $p(\ell)=4/(\lambda \epsilon)$, we know $p(\ell)=O(r / \epsilon + r^2/\epsilon^2 + \kappa )$.
  Condition~\eqref{eq:maindriftcondition} of
  Theorem~\ref{theo:orig-drift} has been established along with these
  bounds on~$p(\ell)$ and $\lambda=\lambda(\ell)$.
	

  To bound the probability of a success within $L(\ell)$ steps, we
  still need a bound on $D(\ell)=\max\{1,\E(e^{-\lambda(X_{t+1}-b)}\mid
  X_t\ge b)\}$. If $1$ does not maximize the expression then 
	 \begin{align*}
   D(\ell) & \;=\; \E(e^{-\lambda(X_{t+1}-b)}\mid X_t\ge b)
  \;\le\; \E(e^{-\lambda\lvert \Delta\rvert}\mid X_t\ge b) \\
  &\, 
  \;\le \; 1+ \E(e^{\gamma \Delta'}\mid X_t\ge b),
  \end{align*}
  where the first inequality follows from $X_t\ge b$ 
  and 
  the second one from $\gamma\ge \lambda$ along with 
	the bound $+1$ for the positive terms as argued above.  The last term can be bounded 
  as in the above calculation leading to $\E(e^{\gamma \Delta'})=O(1)$ since 
  that estimation uses only the second condition, which holds 
  conditional on $X_t > a$. Hence, in any case $D(\ell) =
  O(1)$.  
  Altogether,  we
  have 
  \begin{align*}
  & e^{-\lambda(\ell)\cdot \ell}\cdot D(\ell) \cdot p(\ell) \;\le\; 
  e^{-\lambda\ell} \cdot \frac{4}{\lambda\epsilon} \\
  & \quad \;=\; e^{-\lambda\ell\epsilon + \ln(4/(\lambda\epsilon))}
  \end{align*}

	By the third condition, we have  
	$\lambda \ell \ge 2\ln(4/(\lambda \epsilon))$, 
  which finally means that 
  \[
  e^{-\lambda(\ell)\cdot \ell}\cdot D(\ell) \cdot p(\ell) \;\le\; 
  O(e^{-\lambda\ell\epsilon/2 )) } 
  \]
  Choosing
  $L(\ell)=e^{\lambda\ell/4}$, Theorem~\ref{theo:orig-drift} yields
  \[
  \Prob(T(\ell)\le L(\ell)) \;\le\; L(\ell)\cdot O(e^{-\lambda\ell/4}) \;=\;
  O(e^{-\lambda\ell/4}),\]
  which proves the theorem.  \qed\end{proof}

\bibliographystyle{plainnat}


\end{document}